%% file: main.tex
\documentclass{article} 
\usepackage{iclr2026_conference,times}

\usepackage[T1]{fontenc}
\usepackage[utf8]{inputenc}  
\usepackage{hyperref}
\usepackage{url}
\usepackage{amsmath,amssymb,amsthm}
\usepackage{bm}
\usepackage{booktabs}
\usepackage{enumitem}
\usepackage{graphicx}
\usepackage{wrapfig}
\usepackage{microtype}
\usepackage{titletoc}
\usepackage{multirow}
\usepackage{algorithm}
\usepackage{algpseudocode}
\usepackage{xcolor}
\usepackage{tikz}
\usepackage{listings}
\definecolor{codeback}{rgb}{0.965,0.965,0.945}
\definecolor{codecomment}{rgb}{0,0.5,0}
\definecolor{codekw}{rgb}{0.55,0,0.55}
\definecolor{codestr}{rgb}{0.25,0.5,0.75}
\definecolor{promptrole}{rgb}{0.10,0.24,0.60}
\definecolor{prompttag}{rgb}{0.80,0.37,0.0}
\lstdefinestyle{promptbox}{
  basicstyle=\ttfamily\scriptsize,
  backgroundcolor=\color{codeback},
  frame=single,
  breaklines=true,
  columns=fullflexible,
  keepspaces=true,
  showstringspaces=false,
  captionpos=b,
  aboveskip=4pt, belowskip=2pt,
  xleftmargin=4pt, xrightmargin=4pt,
  morecomment=[l]{\#},
  commentstyle=\color{codecomment},
  morekeywords={def,return,None,True,False},
  keywordstyle=\color{codekw}\bfseries,
  emph={SYSTEM,USER,MODEL,OUTPUT},
  emphstyle={\color{promptrole}\bfseries},
  emph={[2]TASK,ENV,CONTRACT,METRIC,MENU,SANDBOX,BASIS,DIVERSITY,PREFERENCE,EVENT,DICT,STATE,SCENE,VARIABILITY,FIXED,FITNESS,FROZEN,SKILLS,OBJECTIVE,MAP,ROAD,NETWORK,DEMAND,ORDERS,SIMULATED,TIME,DEPLOYMENT,SUPPLY,FLEET,READ,REQUIREMENT,CANDIDATES},
  emphstyle={[2]\color{prompttag}\bfseries},
}
\lstdefinestyle{pybox}{
  language=Python,
  basicstyle=\ttfamily\scriptsize,
  backgroundcolor=\color{codeback},
  commentstyle=\color{codecomment},
  keywordstyle=\color{codekw}\bfseries,
  stringstyle=\color{codestr},
  frame=single,
  breaklines=true,
  showstringspaces=false,
  captionpos=b,
  aboveskip=4pt, belowskip=2pt,
  xleftmargin=4pt, xrightmargin=4pt,
}
\newtheorem{proposition}{Proposition}
\newtheorem{lemma}[proposition]{Lemma}

\newtheorem{corollary}[proposition]{Corollary}

\newcommand{\E}{\mathbb{E}}
\newcommand{\R}{\mathrm{R}}

\newtheorem{definition}{Definition}
\newtheorem{requirement}{Requirement}

\title{Is Per-Agent Policy Composition Safe? Rethinking Successor-Feature Transfer in Cooperative Multi-Agent Reinforcement Learning}

\author{Zijian Zhao$^1$, Sen Li$^{1,2}$ \thanks{Corresponding Author: Sen Li} \\
$^1$The Hong Kong University of Science and Technology \\
$^2$The Hong Kong University of Science and Technology (Guangzhou)
}
\begin{document}
\maketitle

\begin{abstract}
Many reinforcement learning systems, from fleet management to traffic signal control, must serve an objective that changes dynamically after deployment, and retraining a policy for each new objective is prohibitively expensive. For a single agent, this problem is well understood: successor features with generalized policy improvement, together with their universal extension, recombine a library of learned policies into a policy for any new objective, with a guarantee that the result is never worse than any policy in the library. However, multi-agent transfer has received far less attention, and the common practice of letting each agent recombine its own library independently inherits the recipe but not the guarantee. We prove that this independent composition can produce joint behavior strictly worse than every policy in the library, because recombining teammates changes the environment each agent faces and invalidates the values it relies on, a failure with no single-agent counterpart. We further show that the only unconditionally safe fixed rule is synchronized composition, which moves the whole team to one jointly trained policy but cannot serve objectives that assign different goals to different agents. To attain safety and flexibility at once, we propose multi-agent universal successor feature approximators (MA-USFA), a hierarchical method with two layers: a lower layer of universal successor feature approximators that predicts each agent's successor features while conditioned on its teammates' objectives, and an upper composer that selects, across agents, which library entry each agent should follow and supplies the cross-agent correction a per-agent value cannot represent. Trained once over the distribution of objectives, it is applied at deployment with no per-task adaptation. On a controlled grid world and a real-world city-scale traffic signal control problem, MA-USFA matches or exceeds every fixed composition rule and recovers the performance of policies retrained from scratch. The code of this paper is provided at \url{https://github.com/RS2002/MA-USFA}.
\end{abstract}

\section{Introduction}

Reinforcement learning is often deployed in settings where the environment and the available actions stay fixed while the objective changes repeatedly. For example, a traffic signal controller regulates the same intersections with the same phases every day, yet the goal it should optimize shifts with the hour (e.g. peak hour and off-peak hour), and different intersections in the same network may weigh throughput, delay, and queue length differently at the same moment (e.g. main road and branchway). The practical difficulty is that the two conventional responses each pay a price. One trains a dedicated policy for every objective, as cooperative multi-agent methods do when they factorize value for a single fixed team reward~\citep{sunehag2018value,rashid2018qmix,wei2019colight}; this is accurate for the objective it targets but must be redone whenever the preference changes, and the cost grows with the number of agents. The other trains a single policy on one or a few objectives and reuses it on the rest, trading away per-objective performance for coverage~\citep{dasilva2019survey}. The real goal is neither of these but to serve an objective that keeps changing without paying either price.

For a single agent this problem is well understood: successor features (SF) and generalized policy improvement (GPI) decouple the dynamics of a policy from the objective it serves, so that a library of policies learned for a few objectives can be recombined by a closed-form rule into a policy for any new objective in the same family, with a guarantee that the recombined policy is never worse than any policy already in the library~\citep{barreto2017successor,barreto2020fast}, and universal successor feature approximators (USFA), which inherit the goal-conditioned value modeling of universal value function approximators (UVFA)~\citep{schaul2015universal}, extend the same rule from a finite library to a whole distribution of objectives~\citep{borsa2019universal}. Cooperative multi-agent transfer, where the cost of retraining grows with the number of agents, has received far less attention; the natural way to carry the single-agent recipe into a team is to let each agent apply the recombination rule to its own library, selecting its own component of the joint action independently. Several recent works follow exactly this template~\citep{dealmeida2024knowledge,liu2022efficient,nigam2026zeroshot}, on the implicit assumption that the single-agent guarantee carries over to the team; some acknowledge that the multi-agent improvement guarantee is not established, or allow a few steps of fine-tuning at deployment, but none characterizes when the assumption holds and when it fails. (A detailed discussion of related work is deferred to Appendix~\ref{app:related}.)

We supply the missing characterization, and it has two sides. Independent per-agent composition is not safe in general: even in cooperative tasks whose rewards are fully separable, the independently composed policy can be strictly worse than every policy in the shared library, because each agent's values were learned while its teammates followed their old policies, and recombining the library changes the environment each agent faces. This failure has no counterpart in single-agent problems, where the environment is fixed. The rule that does carry a guarantee is synchronized composition, in which the whole team switches together to a single jointly trained policy indexed by one shared choice; we prove that it is unconditionally safe, never worse than any joint policy in the library for any objective, and without any assumptions on the reward or the dynamics. We further give conditions, checkable from the library alone, under which independent composition regains its guarantee, so the field is left with both a correct safe baseline and a precise test for when the cheaper rule may still be used.

Synchronized composition, however, guarantees only that the team performs no worse than the library; because it can only reproduce joint policies stored as a whole, it cannot serve objectives that are heterogeneous across the team, assigning different goals to different agents. We therefore propose multi-agent universal successor feature approximators (MA-USFA), a hierarchical method with two layers. The lower layer is a per-agent USFA that predicts each agent's successor features while conditioned on a compact summary of what the teammates are being asked to do, so that no agent's value assumes a fixed set of teammates; the upper layer is a small learned composer that reads the joint state and selects each agent's library component, learning exactly the cross-agent correction that a per-agent value cannot represent by itself. The composer is trained once over the distribution of objectives anticipated before deployment and applied with no per-task adaptation. On a controlled grid world and a real-world city-scale traffic signal problem with nearly two hundred agents, MA-USFA improves throughput and delay over both fixed rules and recovers the performance of policies retrained separately for each objective. Together these results turn independent per-agent composition, the field's de facto but unexamined heuristic, into a well-posed problem: a precise diagnosis of when it fails, a provably safe baseline to fall back on, and a learned method that is safe and flexible at once.

\section{Preliminaries and Problem Formulation}
\label{sec:prelim}

\subsection{Single-agent transfer: successor features, generalized policy improvement, and universal value modeling}
\label{sec:single}

Consider a Markov decision process with state space $S$, action space $A$, transition kernel $P$, discount factor $\gamma$, and a feature function $\phi: S \times A \to \mathbb{R}^d$. The reward is linear in the features, $r_w(s,a) = \phi(s,a)^\top w$, so different objectives in the same family differ only in the weight vector $w$. The successor features of a policy $\pi$ are the expected discounted feature sums
\begin{equation}
\psi^\pi(s,a) = \E\Big[\sum_{t=0}^{\infty}\gamma^t \phi(s_t, a_t) \,\Big|\, s_0=s,\ a_0=a,\ \pi\Big],
\label{eq:sf}
\end{equation}
and the value of $\pi$ under any objective $w$ is the linear function $V_w^\pi(s) = \psi^\pi(s, \pi(s))^\top w$: a single successor feature model prices a policy under any $w$ without new rollouts.

GPI turns this pricing rule into a composition rule over a library of policies $\{\pi^1, \dots, \pi^K\}$ with successor features $\{\psi^1, \dots, \psi^K\}$: the composed policy acts greedily with respect to the best library value,
\begin{equation}
\pi(s) \in \arg\max_{a} \max_{k} \psi^k(s,a)^\top w_{\mathrm{test}},
\label{eq:gpi}
\end{equation}
and satisfies $V_{w_{\mathrm{test}}}^{\pi}(s) \geq \max_k V_{w_{\mathrm{test}}}^{\pi^k}(s)$ for every state $s$: the composed policy is never worse than any policy in the library~\citep{barreto2017successor,barreto2020fast}. Two premises make this machinery work: the dynamics are fixed, so the successor features stored in the library remain valid whatever the objective; and there is a single decision maker, so the rule in Eq. \eqref{eq:gpi} is executed on the only action space there is.

UVFA make the value model a function of the objective itself, $V(s, g)$, so that one network represents a whole family of value functions and generalizes to unseen objectives by interpolation rather than by retraining a separate value function for each~\citep{schaul2015universal}. USFA combine this idea with successor features in a single model $\tilde\psi(s, a, z, w)$ that carries two axes: a policy axis $z$ that indexes which library policy is being evaluated, and a task axis $w$ that shapes the behavior during training and prices any objective at test time. After one training run over a distribution of objectives, composition at test time is therefore a dot product over a candidate set, with no per-objective adaptation~\citep{borsa2019universal}. With a single goal as the candidate set USFA reduces to UVFA, and with a finite library to SF and GPI.

Appendix~\ref{app:single} recalls the pricing identity and the GPI improvement guarantee, and explains why UVFA and USFA generalize across objectives and what motivates the USFA policy-task factorization; our multi-agent analysis builds directly on these single-agent facts.

\subsection{Cooperative multi-agent setting}
\label{sec:multi}

We now lift the single-agent setup of Section~\ref{sec:single} to a team, and the notation carries an agent index $i$ throughout. We formalize the team as a multi-agent Markov decision process: a set of $N$ agents shares a state space $S$, agent $i$ takes action $a_i$, and the joint action is $a = (a_1, \dots, a_N)$, with $a_{-i} = (a_j)_{j \neq i}$ the teammates' actions. The transition kernel $P(s' \mid s, a)$ and the features depend on the joint action.

Our scope is cooperative transfer, where each agent carries its own objective but the team shares one criterion. Agent $i$ has features $\phi_i(s, a)$ and a task weight $w_i$, giving a per-agent reward $r_i(s,a) = \phi_i(s,a)^\top w_i$ and, under a joint policy $\pi$, a per-agent reward-to-go $V_i^\pi(s) = \psi_i^\pi(s, \pi(s))^\top w_i$ built from the per-agent successor features $\psi_i^\pi$ of Eq. \eqref{eq:sf}. The system objective is the sum of these reward-to-go values,
\begin{equation}
V^\pi(s) \ =\ \sum_{i=1}^{N} V_i^\pi(s) \ =\ \psi^\pi(s, \pi(s))^\top w, \qquad \psi^\pi = \sum_i \psi_i^\pi,
\label{eq:teamval}
\end{equation}
so the team maximizes total reward-to-go, the cooperative criterion. A homogeneous task assigns all agents one shared weight $w$; a heterogeneous task a per-agent vector $(w_1, \dots, w_N)$, the case single-agent transfer has no analogue for. The objective distribution the team is trained and deployed on spans both types. This matters for composition: a synchronized joint policy commits the whole team to one shared weight, so it can be trained on and can serve only the homogeneous objectives, whereas the heterogeneous objectives are reachable only by composing per-agent policies.

Two differences from the single-agent case of Section~\ref{sec:single} drive everything that follows, and each breaks one of its two premises. First, execution is decentralized: $N$ agents each select their own $a_i$, so the single decision maker of Eq. \eqref{eq:gpi} is replaced by $N$ of them acting in parallel. Second, the teammates are part of each agent's environment: the marginal transition of agent $i$ is $P_i(s'_i \mid s, a_i, a_{-i})$ and its features may be $\phi_i(s_i, a_i, a_{-i})$, so recomposing the library changes $a_{-i}$ and hence the very process against which agent $i$'s successor features were measured. The fixed-dynamics premise and the single-decision-maker premise therefore both fail, and Section~\ref{sec:theory} traces every safety question back to these two failures.

\subsection{Libraries and composition rules}

A policy library $\Pi = \{\pi^1, \dots, \pi^K\}$ is trained before deployment and contains two kinds of entries: synchronized entries, joint policies trained on homogeneous tasks so the whole team can switch to $\pi^k$ through a single index $k$, and independent entries, assembled from per-agent policies $\pi^k = (z^k_1, \dots, z^k_N)$ trained for their own task contexts, which span combinations no synchronized entry can represent. Each entry carries a joint successor feature $\psi^k$, measured directly for synchronized entries and summed from per-agent features when the features decompose additively. For each agent $i$ we also define the per-agent successor feature $\psi^k_i(s, a_i)$ of entry $k$: the expected discounted features of agent $i$ when it takes $a_i$ while the teammates follow entry $k$ and everyone continues along entry $k$ afterwards, a snapshot quantity measured against a specific version of the teammates; Requirement~\ref{req:valid} will turn precisely this dependence into a problem. A composition rule $\R$ maps the library and a deployment objective to a joint policy, with no gradient updates at deployment time. This paper studies exactly three, which differ only in how each agent's library component is selected: two fixed rules, analyzed in Section~\ref{sec:theory}, and the learned rule we propose in Section~\ref{sec:method}. ( Table~\ref{tab:spectrum} (Appendix~\ref{app:toy}) summarizes the three rules.)

The first fixed rule, \emph{synchronized composition}, selects one library entry per state and replays its joint action,
\begin{equation}
\pi^{\mathrm{sync}}(s) \ =\ \pi^{k^*}(s), \qquad k^* \in \arg\max_{k} \psi^k\!\big(s, \pi^k(s)\big)^\top w_{\mathrm{test}},
\label{eq:l1}
\end{equation}
the composition rule of Eq. \eqref{eq:gpi} restricted to the joint actions the entries themselves take; under valid values the score in Eq. \eqref{eq:l1} is the value of the entry itself, $\psi^k(s, \pi^k(s))^\top w_{\mathrm{test}} = V^{\pi^k}(s)$, so the synchronized rule always replays the entry that is best at the current state. 

The second fixed rule, \emph{independent composition}, lets each agent apply the single-agent rule to its own library,
\begin{equation}
\pi^{\mathrm{ind}}_i(s) \in \arg\max_{a_i} \max_{k} \psi^k_i(s, a_i)^\top w_i,
\label{eq:l2}
\end{equation}
which is how the recent multi-agent transfer literature carries GPI into a team~\citep{dealmeida2024knowledge,liu2022efficient,nigam2026zeroshot}. 

The third learned rule, \emph{MA-USFA}, is our proposed method: it keeps the decision space of independent composition but replaces the fixed argmax in Eq. \eqref{eq:l2} with a learned upper-layer composer over a per-agent successor-feature layer (Section~\ref{sec:method}). The deployment goal is zero per-task adaptation, amortized over the preference distribution: all learning happens once, before deployment, and at deployment the rule acts on the frozen library with no adaptation to the particular $w_{\mathrm{test}}$.

\section{Fixed Methods: Synchronized and Independent Compositions}
\label{sec:theory}


\subsection{Design goals: safety and flexibility beyond the library}
\label{sec:goals}

To evaluate a composition rule, we consider two criteria, and no fixed rule delivers both. The first is \emph{safety}: the composed policy should never be worse than the best policy already in the library.

\begin{definition}[Safety]
\label{def:safety}
A composition rule $\R$ is safe on $(\Pi, W_{\mathrm{test}})$ if the composed policy $\pi^{\R}$ satisfies, for every $w_{\mathrm{test}} \in W_{\mathrm{test}}$ and every state $s$,
\begin{equation}
V_{w_{\mathrm{test}}}^{\pi^{\R}}(s) \geq \max_{k} V_{w_{\mathrm{test}}}^{\pi^k}(s),
\label{eq:safety}
\end{equation}
where the value is the team objective $V = \sum_i V_i$ of Eq. \eqref{eq:teamval} and the maximum is over library entries with valid joint successor features.
\end{definition}

The second goal is \emph{flexibility} beyond the library: a heterogeneous objective can ask each agent for something that no single stored joint policy contains, so we want a rule that can serve such objectives and, where the task structure allows, improve on every stored entry rather than merely tie the best one. 

The two goals pull against each other, and the two fixed rules sit at opposite corners: synchronized composition, next, meets the safety goal but forfeits flexibility; independent composition reaches for flexibility but forfeits safety; MA-USFA (Section~\ref{sec:method}) is built to meet both.

\subsection{Synchronized composition is unconditionally safe}
\label{sec:sync}

\begin{proposition}[Synchronized composition is unconditionally safe]
\label{prop:l1safe}
Let $\Pi$ be any library, $\pi^{\mathrm{sync}}$ the synchronized composition of Eq. \eqref{eq:l1}, and $w_{\mathrm{test}}$ any deployment objective. For every state $s$ and every library entry $k$,
\begin{equation}
V_{w_{\mathrm{test}}}^{\pi^{\mathrm{sync}}}(s) \ \geq\ V_{w_{\mathrm{test}}}^{\pi^k}(s).
\end{equation}
\end{proposition}

\begin{proof}[Proof sketch]
At every state the synchronized rule scores each entry by the value of replaying it, $\psi^k(s, \pi^k(s))^\top w_{\mathrm{test}} = V^{\pi^k}(s)$, and plays the action of the highest-scoring entry; comparing against any entry $\pi^k$, its action scores at least as high at every state, so the one-step advantage telescopes into $V_{w_{\mathrm{test}}}^{\pi^{\mathrm{sync}}}(s) \geq V_{w_{\mathrm{test}}}^{\pi^k}(s)$ (Appendix~\ref{app:proofs}).
\end{proof}

The structural reason is that synchronized switching never changes anyone's environment, so every joint successor feature stays valid; synchronized composition is therefore the safe floor any claimed improvement must beat. What it cannot do, taken up in Section~\ref{sec:beyond}, is coverage: its composition space is only the $K$ library entries.

\subsection{Independent composition: unsafe in general, safe under two conditions}
\label{sec:independent}

Independent composition can express the heterogeneous objectives synchronized composition cannot, but its safety is conditional on two requirements, one repairing each of the two single-agent premises that a team breaks (Section~\ref{sec:multi}). Decentralized execution breaks the single-decision-maker premise, so the first requirement constrains which joint actions the per-agent choices select; teammates entering each agent's environment breaks the fixed-dynamics premise, so the second requirement asks whether the stored per-agent values are still correct after recomposition.

\begin{requirement}[Selection alignment]
\label{req:align}
The joint actions optimal for the test objective are reachable by per-agent greedy choices: with $A^{*}(s) = \arg\max_a V^{*}_{w_{\mathrm{test}}}(s, a)$, alignment requires $A^{*}(s) = \prod_i A^{*}_i(s)$, the Individual-Global-Max (IGM) condition of value-decomposition theory~\citep{sunehag2018value,rashid2018qmix,son2019qtran} carried from training time to composition time.
\end{requirement}

\begin{requirement}[Value validity]
\label{req:valid}
The per-agent successor features stored in the library remain the true successor features of the composed joint policy: each $\psi^k_i$ was measured while the teammates followed entry $k$, and once composition changes their behavior $\psi^k_i$ prices a policy that is no longer being executed.
\end{requirement}


\begin{proposition}[The two requirements suffice]
\label{prop:reqsuffice}
Fix a library and a test objective $w_{\mathrm{test}}$. If Requirement~\ref{req:align} and Requirement~\ref{req:valid} both hold, then independent composition is safe: $V_{w_{\mathrm{test}}}^{\pi^{\mathrm{ind}}}(s) \geq \max_{k} V_{w_{\mathrm{test}}}^{\pi^k}(s)$ for every state $s$.
\end{proposition}

The argument is direct: validity makes each per-agent score $\psi^k_i(s, a_i)^\top w_i$ the true value of the composed policy, and alignment makes the per-agent greedy choice coincide with the joint-GPI action, so independent composition executes the joint-GPI rule with correct values and inherits its single-agent guarantee (Appendix~\ref{app:proofs}). The two requirements are thus the whole story, and what remains is to ask when each holds. They are fed by two independent channels of coupling: the reward channel (cross-features in $\phi$) governs Requirement~\ref{req:align}, and the transition channel (teammate actions in the dynamics) governs Requirement~\ref{req:valid}. We show first that a benign-looking task can break validity even when alignment holds, then give the exact condition on each channel under which the independent rule is not merely safe but no worse than synchronized composition.

\paragraph{It can fail even when rewards are fully separable.}
\begin{lemma}[Independent composition is not always safe]
\label{lem:l2unsafe}
For every number of agents $N \geq 2$ there exists a cooperative multi-stage task whose rewards are fully separable (each agent's reward depends only on its own action), such that the alignment condition (Requirement~\ref{req:align}) holds at the test objective and every library entry is suboptimal, yet
\begin{equation}
V_{w_{\mathrm{test}}}^{\pi^{\mathrm{ind}}}(s_0) \ <\ \max_{k} V_{w_{\mathrm{test}}}^{\pi^k}(s_0),
\end{equation}
so the independently composed policy is strictly worse than every policy in the library, while the centralized joint-GPI rule over the same library attains the joint optimum $V^{*}_{w_{\mathrm{test}}}(s_0)$. The failure is therefore specific to decentralization.
\end{lemma}

The construction is a multi-stage task with per-agent feature tables (Appendix~\ref{app:proofs}). The failure is not misalignment (Requirement~\ref{req:align} holds) but validity (Requirement~\ref{req:valid} fails). Each agent acts exactly as its library values tell it to, but the values are stale: agent $i$ scores its choice against the features it would accrue if the teammates kept the library's actions, while their independent choices send the team down a different branch and the accrued features fall short. A single stale rating pushes the team below its own best library entry, the precise sense in which teammates constitute the environment. Whether the failure is statistically visible depends on how the library was built, so safety must be characterized structurally rather than by example.

\paragraph{The reward channel: supermodularity and the weight cone.}

\begin{proposition}[Supermodular values make independent composition no worse than synchronized composition]
\label{prop:l2supermod}
Let $g_s(a) = \max_k \psi^k(s, a)^\top w_{\mathrm{test}}$ be the joint GPI value of the library. Suppose that for every state $s$, $g_s$ is supermodular on the lattice $\{0,1\}^N$, and ties are broken consistently across agents. Then the independent composition \eqref{eq:l2} satisfies, for every state $s$,
\begin{equation}
V_{w_{\mathrm{test}}}^{\pi^{\mathrm{ind}}}(s) \ \geq\ V_{w_{\mathrm{test}}}^{\pi^{\mathrm{sync}}}(s).
\end{equation}
\end{proposition}

\paragraph{Mechanism.} Supermodularity of $g_s$ says that the marginal gain of one agent improving its action increases in the actions of the other agents; by Topkis's monotone comparative statics~\citep{topkis1998supermodularity} the argmax set is then a sublattice of the action lattice, and with a consistent tie-break the per-agent greedy choices coincide with the joint greedy choice. Requirement~\ref{req:align} therefore holds as a condition on the value structure rather than on an assumed decomposition, the structural generalization of value decomposition, which guarantees alignment by construction during training~\citep{sunehag2018value,rashid2018qmix,son2019qtran} whereas supermodularity grants it at composition time.

\begin{corollary}[The weight cone]
\label{cor:cone}
Assume the per-pair value differences of the library policies are linear in the features. Then the per-state supermodularity condition of Proposition~\ref{prop:l2supermod} holds if and only if the test weight lies in the cone
\begin{equation}
K_{\phi} \ =\ \big\{ w : \ \textstyle\sum_{d} w_d\, \Delta\phi_d(a, b) \geq 0 \ \text{ for every pair of joint actions } a, b \big\},
\end{equation}
where $\Delta\phi_d(a, b) = \phi_d(a \vee b) + \phi_d(a \wedge b) - \phi_d(a) - \phi_d(b)$ is the supermodularity gap of feature $d$ on the pair. A sufficient and easy-to-use condition: every feature supermodular and $w \geq 0$.
\end{corollary}

The cone is checkable from library statistics: membership of $w_{\mathrm{test}}$ is a dot product, so whether the free region applies to a given task can be read off the library and the test weight without running the policy.

\paragraph{The transition channel: factorization is necessary, not sufficient.}

\begin{proposition}[Factored transitions with per-agent rewards]
\label{prop:l2factored}
Suppose the transition kernel factorizes per agent, $s'_i = f_i(s_i, a_i)$, so that no agent's dynamics depend on teammate actions, and the reward features decompose per agent, $\phi(s, a) = \sum_i \phi_i(s_i, a_i)$. Then the per-agent successor feature $\psi^k_i(s, a_i)$ is the discounted feature stream of agent $i$'s own dynamics, and it depends on neither the entry $k$ nor the composition. Requirement~\ref{req:valid} therefore holds by construction, for every library and every recomposition. Under the supermodularity and tie-breaking conditions of Proposition~\ref{prop:l2supermod}, the independent composition then satisfies $V_{w_{\mathrm{test}}}^{\pi^{\mathrm{ind}}}(s) \geq V_{w_{\mathrm{test}}}^{\pi^{\mathrm{sync}}}(s)$ for every state $s$.
\end{proposition}

Factorization neutralizes the transition channel only in the decomposed-reward case, where the per-agent values are exact marginal processes of the agent's own dynamics. When the reward features couple the agents, the cross-feature components are measured against the library's teammate behavior and recomposition invalidates them: even with factored transitions and per-state supermodularity, the independent rule can violate the safety bound under the exact semantics of Eq. \eqref{eq:l2}. 

\subsection{Beyond the free region: when no fixed rule suffices}
\label{sec:beyond}

The two fixed rules cover complementary but incomplete parts of the problem, and a large class of tasks falls outside both. Synchronized composition is always safe but cannot serve heterogeneous objectives: it replays one library entry as a joint action, so its reachable behavior is exactly the $K$ entries, and a test objective that asks different agents to pursue different components of the library has no synchronized representative. Independent composition can express those heterogeneous choices, but only inside the free region of Propositions~\ref{prop:l2supermod} and \ref{prop:l2factored}, and both free conditions can fail. Transition coupling breaks Requirement~\ref{req:valid} outright: the counterexample of Lemma~\ref{lem:l2unsafe} is an exact failure in which the independent rule falls below the best library entry and below joint-GPI. Cross-feature rewards break Requirement~\ref{req:align}, so even factored transitions do not by themselves restore safety.

The two conditions are structural properties of the task, not tunable knobs. For example, a traffic network couples each intersection to its upstream neighbors, so factorization fails by construction and real coupled problems live in the failure region, where synchronized composition is too rigid to serve heterogeneous demand and independent composition is no longer safe. This is the region a learned composition rule must cover, and it is the design target of MA-USFA (Section~\ref{sec:method}), which is applied to every objective and is built to be at least as good as either fixed rule everywhere, so its use never depends on a safety test. The theory of this section instead governs the fixed-rule-only regime: when a learned composer is unavailable, synchronized composition is the unconditionally safe default, and the independent rule may replace it precisely where its two requirements hold.

\section{Learnable Method: MA-USFA}
\label{sec:method}

\subsection{Overview}

Section~\ref{sec:theory} shows that each fixed rule is safe only on part of the objective space, so a single rule that is safe and flexible on every objective must learn the composition.
MA-USFA is a hierarchy of two decision layers over the cooperative multi-agent MDP of Section~\ref{sec:multi} (Fig.~\ref{fig:arch}). The lower layer acts in that MDP with the primitive per-agent actions $a_i$: a per-agent context-conditioned USFA predicts each agent's successor features, conditioned on a compact summary of the teammates' task context, so that no agent's value assumes a fixed set of teammates. The upper layer is a selection layer stacked on top, and it is where composition happens. At each state its observation is the joint state $s$, the task vector $w$, and the lower layer's candidate valuations $\{\tilde\psi_i(s, \cdot, z_i)^\top w_i : z_i \in C_i\}$; its action is a per-agent choice of library entry $g_i \in C_i$, whose primitive action the team then executes. The upper layer thus decides which library policy each agent follows, the cross-agent correction a per-agent value cannot express, while the lower layer supplies the values it chooses among. Because the composer is initialized at the independent rule and trained with the value layer frozen, moving only in directions that raise team value, on every objective MA-USFA starts from a fixed rule and only improves on it, at least matching it everywhere and exceeding it where the correction has something to add.

\subsection{The two layers}

\paragraph{Lower layer: per-agent context-conditioned successor features.}

Each agent $i$ learns a value model
\begin{equation}
\tilde\psi_i(s, a_i \mid z_i, w_{-i}),
\label{eq:lowersf}
\end{equation}
where $z_i$ is the agent's own policy encoding and $w_{-i} = (w_j)_{j \neq i}$ collects the teammates' task weights. The two arguments play the two roles of USFA~\citep{borsa2019universal}: $z_i$ is the policy axis, a learned index over the agent's library entries, and the weights carry the task axis, linear in the value and applied only at test time. The conditioning on $w_{-i}$ is the new ingredient, the architectural answer to Requirement~\ref{req:valid}: each teammate's weight indexes a cluster of teammate policies, so agent $i$ learns expected successor features over the teammate behavior distribution rather than a snapshot against one fixed teammate version; without it the value is stale the moment the library holds more than one teammate policy, the failure Lemma~\ref{lem:l2unsafe} exposes. For large teams the context is summarized (mean or learned embedding of $w_{-i}$) so the model does not scale with $N$.

Training follows the USFA protocol: the TD target of $\tilde\psi_i$ is the feature vector $\phi$, so the update does not depend on $w$, which plays the standard roles of behavior anchor, library anchor ($z \sim D_z(\cdot \mid w)$), and test-time pricing. With per-agent additive features the joint successor feature of a synchronized entry decomposes as $\psi^{\mathrm{joint},k} = \sum_i \psi^k_i$, so one network serves both the synchronized anchor and the per-agent library (Appendix~\ref{app:usfa}).

\paragraph{Upper layer: the learned composer.}

The composer is a collection of per-agent selectors $\Upsilon^\theta = (\Upsilon^\theta_1, \dots, \Upsilon^\theta_N)$. Its input is the joint state, the task vector, and the candidate valuations $\{\tilde\psi_i(s, \cdot, z_i)^\top w_i : z_i \in C_i\}$ from the lower layer; its output is a selection $g_i \in C_i$ per agent:
\begin{equation}
g_i \;=\; \Upsilon^\theta_i\!\left(s,\, w \;\middle|\; \big\{\,\tilde\psi_j(s,\cdot,z_j)^\top w_j \,\big\}_{j,\, z_j \in C_j}\right)\ \in\ C_i,
\label{eq:uppersel}
\end{equation}
It learns the cross-agent value correction: the lower layer absorbs the expected part of the teammate dependence, predictable from the public task context, and the composer adds the residual that depends on the joint combination of candidate policies, which no per-agent value can price. It does not relearn single-agent decision making, only corrects values across agents, which is why it is small and trainable on a fraction of the library's budget. Similar to the lower layer, for large teams the contexts ($w$ and $\big\{\,\tilde\psi_j(s,\cdot,z_j)^\top w_j \,\big\}_{j,\, z_j \in C_j}$) are summarized (e.g. via attention or communication mechanisms) to ensure scalability.

\begin{wrapfigure}{r}{0.25\textwidth}
\centering
\vspace{-0.6\baselineskip}
\includegraphics[width=0.25\textwidth]{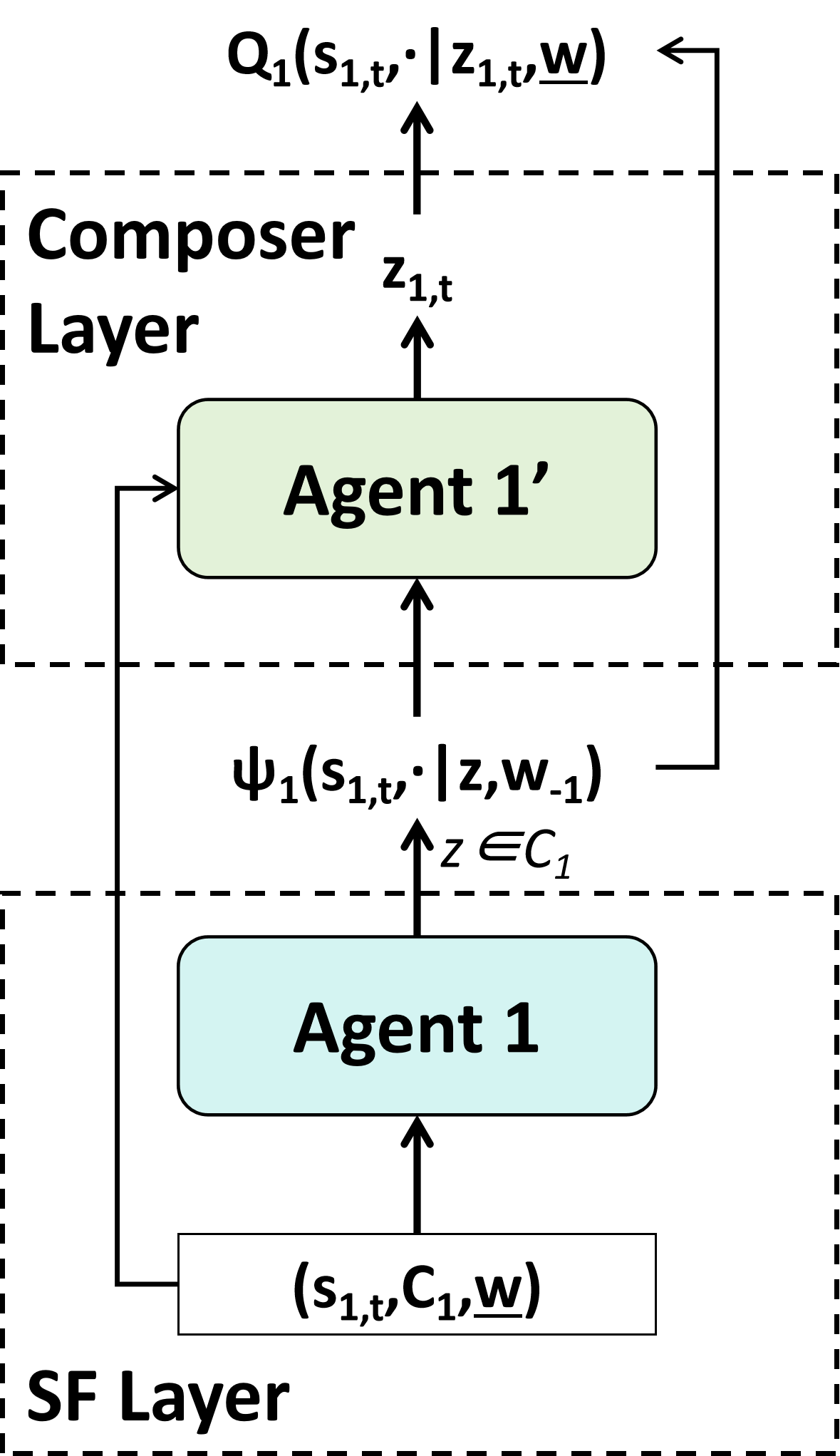}
\caption{MA-USFA shown for a single agent.}
\label{fig:arch}
\vspace{-0.8\baselineskip}
\end{wrapfigure}

\subsection{Training and deployment}

Training has two phases, one per layer (Table~\ref{tab:phases}, Appendix~\ref{app:method}; full pseudocode in Appendix~\ref{app:algo}). Phase 1 trains the lower per-agent value layer, sampling tasks at random from the homogeneous and heterogeneous scenarios with the $w_{-i}$ conditioning, and then freezes it; this layer produces the candidates the composer selects among. Phase 2 trains the upper composer, initialized at the independent transfer policy with the value layer frozen; wherever that policy is already optimal the value-increasing gradient vanishes and no parameters move, so the budget concentrates on the objectives whose coupling actually calls for a cross-agent correction, namely coupled dynamics and cross-feature rewards. This is what keeps the composer small, the difference between MA-USFA and an end-to-end composer.

Deployment is a single forward pass with no gradient updates, run on every objective: the lower layer prices the candidates by dot product and the composer selects, so any $w_{\mathrm{test}}$ is served with no per-task adaptation. This is the entire method at test time. The composer is always the deciding rule, and the fixed rules of Section~\ref{sec:theory} take no part in it. The cost of learning to compose is paid once before deployment over the distribution of anticipated objectives, after which every objective needs only forward passes, the precise sense in which the method needs zero per-task adaptation.

\section{Experiments}
\label{sec:exp}

\subsection{Experiment Setup}
In this section, we aim to answer three questions: (i) Does independent composition fail exactly where the theory says it must, while synchronized composition holds its floor (Proposition~\ref{prop:l1safe}, Lemma~\ref{lem:l2unsafe})? (ii) Does the learned composer recover per-objective retraining where the fixed rules break? (iii) And do both survive at large scale scenario, where coupling is endogenous rather than dialed in? The controlled domain below answers the first two. The third is answered by a traffic signal control task on a Manhattan grid of $28 \times 7$ intersections ($196$ signal agents), run on a homogeneous task (one shared weight) and a heterogeneous task (each intersection with its own weight). Overall, the results show MA-USFA dominates both fixed rules on every metric, with the largest margins on the heterogeneous task, and we leave it to Appendix~\ref{app:tsc} due to page limitation.

\paragraph{Methods and baselines.} We compare five methods on every objective. The three composition rules of Section~\ref{sec:prelim} are the objects of study: synchronized composition, independent composition, and MA-USFA. Two methods bracket them: joint-GPI, centralized generalized policy improvement over the joint action space (the prior-work baseline, feasible only for small teams), and per-task retraining, trained from scratch on the deployment reward including the out-of-basis collision penalty (the oracle ceiling MA-USFA aims to match).

\subsection{Controlled domain: SFWorld}

SFWorld is a multi-agent grid world that exposes the two coupling channels of Section~\ref{sec:theory} as controllable parameters. $N$ agents move on a $5 \times 5$ grid with five actions over $40$ steps ($\gamma = 0.95$); features $\phi_i(s')$ are smooth bumps over $K$ resource cells, and the reward $r_i = \phi_i(s')^\top w_i - 2.0 \cdot \mathrm{blocked}_i$ adds a collision penalty outside the feature basis that any closed-form pricing $\psi \cdot w$ is structurally blind to. The coupling parameter $\kappa$ is the probability that two agents targeting the same cell collide: at $\kappa = 0$ the dynamics factorize (Proposition~\ref{prop:l2factored}); as $\kappa$ grows, each agent's kernel depends on the teammates' actions. In $A_{\mathrm{distinct}}$ the optimal regions are disjoint so coupling never binds; in $B_{\mathrm{overlap}}$ the agents contend for a shared region so collisions arise whenever the joint policy hybridizes. More details are provided at Appendix~\ref{app:toy}.

\begin{table}[htbp]
\centering
\caption{Controlled domain, $N = 2$, four-entry corner library. Team return over $40$ steps, mean $\pm$ standard deviation over three evaluation seeds; $\mathrm{coll_{ind}}$ is the collision rate of the independent rule. Per row, \textbf{best} is bold and \underline{second best} underlined (used throughout the paper). }
\label{tab:mainbody}
\resizebox{\textwidth}{!}{%
\begin{tabular}{@{}llrrrrrr@{}}
\toprule
task & $\kappa$ & Sync & Indep & joint-GPI & MA-USFA & retrain & $\mathrm{coll_{ind}}$ \\
\midrule
$A_{\mathrm{distinct}}$ & 0.00 & $23.42{\scriptstyle\pm0.34}$ & $\mathbf{38.73}{\scriptstyle\pm0.47}$ & $23.74{\scriptstyle\pm0.29}$ & $38.40{\scriptstyle\pm0.53}$ & $\underline{38.69}{\scriptstyle\pm0.48}$ & 0.000 \\
$A_{\mathrm{distinct}}$ & 0.50 & $23.37{\scriptstyle\pm0.16}$ & $\underline{38.23}{\scriptstyle\pm0.01}$ & $23.63{\scriptstyle\pm0.14}$ & $38.17{\scriptstyle\pm0.22}$ & $\mathbf{38.35}{\scriptstyle\pm0.10}$ & 0.000 \\
$A_{\mathrm{distinct}}$ & 1.00 & $23.44{\scriptstyle\pm0.15}$ & $\underline{38.31}{\scriptstyle\pm0.27}$ & $23.80{\scriptstyle\pm0.23}$ & $38.25{\scriptstyle\pm0.18}$ & $\mathbf{38.48}{\scriptstyle\pm0.16}$ & 0.000 \\
\addlinespace
$B_{\mathrm{overlap}}$ & 0.00 & $31.52{\scriptstyle\pm0.26}$ & $\mathbf{38.77}{\scriptstyle\pm0.14}$ & $32.59{\scriptstyle\pm0.20}$ & $\underline{38.44}{\scriptstyle\pm0.13}$ & $37.56{\scriptstyle\pm0.25}$ & 0.000 \\
$B_{\mathrm{overlap}}$ & 0.25 & $31.36{\scriptstyle\pm0.34}$ & $29.96{\scriptstyle\pm0.32}$ & $\mathbf{31.71}{\scriptstyle\pm0.24}$ & $\underline{31.52}{\scriptstyle\pm0.16}$ & $30.56{\scriptstyle\pm0.49}$ & 0.125 \\
$B_{\mathrm{overlap}}$ & 0.50 & $\mathbf{31.27}{\scriptstyle\pm0.09}$ & $21.04{\scriptstyle\pm0.94}$ & $30.65{\scriptstyle\pm0.14}$ & $30.94{\scriptstyle\pm0.30}$ & $\underline{31.08}{\scriptstyle\pm0.25}$ & 0.225 \\
$B_{\mathrm{overlap}}$ & 0.75 & $\mathbf{31.15}{\scriptstyle\pm0.19}$ & $13.67{\scriptstyle\pm0.30}$ & $29.38{\scriptstyle\pm0.29}$ & $\underline{30.71}{\scriptstyle\pm0.25}$ & $30.67{\scriptstyle\pm0.10}$ & 0.318 \\
$B_{\mathrm{overlap}}$ & 1.00 & $\mathbf{31.24}{\scriptstyle\pm0.16}$ & $7.07{\scriptstyle\pm0.20}$ & $28.62{\scriptstyle\pm0.60}$ & $\underline{31.18}{\scriptstyle\pm0.18}$ & $30.33{\scriptstyle\pm0.31}$ & 0.369 \\
\midrule
\multicolumn{2}{@{}l}{Average (all tasks)} & 28.35 & 28.22 & 28.02 & $\mathbf{34.70}$ & $\underline{34.46}$ & 0.130 \\
\bottomrule
\end{tabular}}
\end{table}

Table~\ref{tab:mainbody} carries the whole controlled-domain story: its $B_{\mathrm{overlap}}$ block exhibits the four claims of the theory at once, and $A_{\mathrm{distinct}}$ isolates coverage from coupling. Down $B_{\mathrm{overlap}}$ the synchronized rule stays flat as the coupling climbs (Proposition~\ref{prop:l1safe}), since synchronized switching never changes anyone's environment, whereas the independent rule falls monotonically as its collision rate rises (Lemma~\ref{lem:l2unsafe}), pairing per-agent choices into hybrid joint policies whose collisions the harvest-only values cannot price. The two orderings cross: the independent rule wins in the free region but the order flips once coupling binds. MA-USFA instead tracks the oracle ceiling at every coupling level; joint-GPI degrades only mildly but never reaches retraining and is unavailable beyond small teams. In $A_{\mathrm{distinct}}$, where collisions never occur, independent composition is free as Proposition~\ref{prop:l2factored} predicts. Averaged over all eight tasks MA-USFA is strongest at $34.70$, ahead of retraining ($34.46$) and far above every fixed rule, with team-size and library-coverage sweeps in Appendix~\ref{app:toy} confirming the pattern.

\section{Conclusion}

In this paper, we study transfer learning in cooperative multi-agent reinforcement learning, covering both heterogeneous and homogeneous tasks. We first identify that independent per-agent composition (the approach adopted by most prior work) does not inherit the single-agent safety guarantee. In contrast, the only unconditionally safe fixed rule is synchronized composition, albeit at the cost of coverage. To balance safety and flexibility, we propose MA-USFA, a hierarchical solution that combines a frozen per-agent value layer with a learned cross-agent correction. Trained once over the preference distribution and deployed with zero per-task adaptation, MA-USFA attains both the safety of synchronized composition and the flexibility of independent recombination. Evaluated on a controlled grid-world scenario and a real-world large-scale traffic signal control task, MA-USFA matches or exceeds every fixed composition rule and recovers the performance of per-task retraining.

\section*{Ethics Statement}
This work adheres to the principles outlined in the ICLR Code of Ethics.

\bibliography{refs}
\bibliographystyle{iclr2026_conference}

\appendix
\newpage
\section*{Appendix Contents}  
\startcontents  
\printcontents{}{1}{\setcounter{tocdepth}{2}}  
\newpage

\input{appendices}

\end{document}

%% file: appendices.tex

\section{Related Works}
\label{app:related}

\paragraph{Single-agent transfer.}
Successor features decouple the dynamics of a policy from the objective it serves: a single successor feature model prices any policy under any objective in a linear family by a dot product~\citep{barreto2017successor}. Generalized policy improvement turns this pricing rule into a composition rule with a guarantee: the greedy recombination of a library of policies is never worse than any policy in the library~\citep{barreto2020fast}. Universal value function approximators condition the value model on the goal directly and interpolate across goals~\citep{schaul2015universal}. Universal successor feature approximators merge the two: a model conditioned on a policy encoding and a task weight, trained once over a distribution of objectives, so that composition at test time is a closed-form dot product over a candidate set~\citep{borsa2019universal}. Optimistic linear support and successor features extend the same machinery to sequential policy transfer across a set of objectives~\citep{alegre2022optimistic}. The present paper takes the operating mode of this lineage, train once and compose with no per-task adaptation, and asks what survives in a team; the answer is that the guarantee does not survive independent per-agent execution.

\paragraph{Multi-agent transfer.}
The survey of da Silva and Costa~\citep{dasilva2019survey} organizes multi-agent transfer along what is transferred, to whom, and with what mechanism, and documents that most proposals carry single-agent recipes into teams without re-examining their premises. The works closest to ours follow exactly the template we analyze: \citet{dealmeida2024knowledge} transfer knowledge across multi-objective multi-agent tasks by letting each agent keep its own library and apply generalized policy improvement independently, and do not establish a multi-agent improvement guarantee; \citet{liu2022efficient} transfer successor features per agent to accelerate exploration, again per-agent; \citet{nigam2026zeroshot} study zero-shot coordination in ad hoc teams with generalized policy improvement and difference rewards, and explicitly note that the single-agent improvement guarantee is not established for teams, the gap this paper fills with a precise negative result and a positive one. None of these works characterizes when the per-agent recipe is safe and when it fails; our Lemma~\ref{lem:l2unsafe} and Propositions~\ref{prop:l2supermod} and \ref{prop:l2factored} draw that boundary.

\paragraph{Value decomposition.}
The cooperative deep multi-agent literature aligns per-agent and joint value models at training time. VDN sums the per-agent values~\citep{sunehag2018value}; QMIX and Weighted QMIX mix them through monotone networks and characterize the family of joint value functions representable under monotonicity~\citep{rashid2018qmix,rashid2020weighted}; QTRAN formalizes the individual-global-max (IGM) condition under which per-agent greedy execution equals joint greedy execution~\citep{son2019qtran}. These works guarantee alignment by construction during training, at the price of a fixed task. Our Requirement~\ref{req:align} imports the IGM condition to composition time, where the values are no longer trained together, and our Lemma~\ref{lem:l2unsafe} shows that even when the condition holds, composition can still be unsafe because the per-agent values are stale, the failure that Requirement~\ref{req:valid} names and that no training-time credit assignment mechanism addresses.


\section{Single-Agent Foundations}
\label{app:single}

This appendix recalls the single-agent facts that Section~\ref{sec:single} uses and that the multi-agent proofs of Appendix~\ref{app:proofs} reduce to. Throughout, features are $\phi(s, a) \in \mathbb{R}^d$, a reward is a linear readout $r_w(s, a) = \phi(s, a)^\top w$, and the successor features of a policy $\pi$ are the discounted feature stream $\psi^\pi(s, a) = \mathbb{E}^\pi\!\left[\sum_{t \geq 0} \gamma^t \phi(s_t, a_t) \mid s_0 = s, a_0 = a\right]$.

\paragraph{The pricing identity.}
Because expectation is linear and the reward is linear in the features, the value of $\pi$ under any objective $w$ is a dot product against a single, objective-independent successor feature model:
\begin{equation}
V^\pi_w(s, a) \;=\; \mathbb{E}^\pi\!\Big[\textstyle\sum_{t \geq 0} \gamma^t\, \phi(s_t, a_t)^\top w \,\Big|\, s_0 = s, a_0 = a\Big] \;=\; \psi^\pi(s, a)^\top w .
\label{eq:pricing}
\end{equation}
The identity is what lets one model, learned once, price a policy under every objective in the linear family without re-estimating a value function per objective~\citep{barreto2017successor}; it is the single-agent form of the team identity \eqref{eq:teamval}, and every appearance of $\psi^\top w$ in the main text is an instance of it.

\paragraph{The generalized policy improvement guarantee.}
Given a library $\{\pi^k\}_{k=1}^K$ with successor features $\{\psi^k\}$ and a test objective $w$, the GPI policy acts greedily with respect to the best library value at each state, $\pi^{\mathrm{gpi}}(s) \in \arg\max_a \max_k \psi^k(s, a)^\top w$.

\begin{proposition}[GPI improvement, \citealp{barreto2017successor,barreto2020fast}]
\label{prop:gpi}
For every state $s$ and every library entry $k$, $V^{\pi^{\mathrm{gpi}}}_w(s) \geq V^{\pi^k}_w(s)$.
\end{proposition}

\begin{proof}
Let $Q(s, a) = \max_k \psi^k(s, a)^\top w$. For any $k$, the greedy action satisfies $\max_a Q(s, a) \geq Q(s, \pi^k(s)) \geq \psi^k(s, \pi^k(s))^\top w = V^{\pi^k}_w(s)$, so $\pi^{\mathrm{gpi}}$ has one-step advantage at least that of $\pi^k$ at every state. The advantage telescopes: writing $T^{\mathrm{gpi}}$ for the Bellman operator of $\pi^{\mathrm{gpi}}$, $V^{\pi^k}_w \leq T^{\mathrm{gpi}} V^{\pi^k}_w \leq (T^{\mathrm{gpi}})^n V^{\pi^k}_w \to V^{\pi^{\mathrm{gpi}}}_w$ by monotonicity and the $\gamma$-contraction of $T^{\mathrm{gpi}}$. Proposition~\ref{prop:l1safe} in the main text is the synchronized-composition specialization of this argument to a team that switches together.
\end{proof}

\paragraph{Why UVFA and USFA generalize across objectives.}
Universal value function approximators condition the value directly on the goal, $V(s, g)$, and share parameters across goals, so that a value learned for one goal transfers by function approximation to nearby goals rather than being retrained from scratch~\citep{schaul2015universal}. Universal successor feature approximators keep the pricing identity \eqref{eq:pricing} but make the successor feature model itself universal, $\tilde\psi(s, a, z, w)$, with a policy axis $z$ that indexes which library policy is being priced and a task axis $w$ that shapes the training behavior distribution and prices the result at test time~\citep{borsa2019universal}. The motivation is precisely the operating mode this paper needs: separating the policy axis from the task axis lets a single network be trained once over a distribution of objectives and then, at deployment, evaluate a dot product over a candidate set with no per-task adaptation, because the objective enters only through the linear factor $w$ and never through retraining. MA-USFA is the multi-agent counterpart: it keeps this policy-task factorization per agent and adds a teammate-context axis, so that each agent's successor features remain valid as the teammates' objectives change; Appendix~\ref{app:usfa} states the reuse precisely.

\section{Proofs}
\label{app:proofs}

\subsection{Proof of Proposition~\ref{prop:l1safe}}

Fix the library $\Pi$, the test objective $w_{\mathrm{test}}$, and an initial state $s_0$. Write $v^k(s) = V_{w_{\mathrm{test}}}^{\pi^k}(s)$ for the value of entry $k$ under the test objective, and let $k^*(s) \in \arg\max_k v^k(s)$. The synchronized rule \eqref{eq:l1} plays $\pi^{k^*(s)}(s)$ at state $s$, provided the score $\psi^k(s, \pi^k(s))^\top w_{\mathrm{test}}$ equals the true value $v^k(s)$; this is the validity premise of the proposition, and it holds exactly under the linear feature model. We prove $V_{w_{\mathrm{test}}}^{\pi^{\mathrm{sync}}}(s) \geq v^k(s)$ for every state $s$ and every entry $k$.

For a finite horizon $T$, let $V_t^{\pi}(s)$ denote the value of policy $\pi$ with $t$ steps remaining, and $V_t^{\pi^k}(s)$ likewise. We show by induction on $t$ that $V_t^{\pi^{\mathrm{sync}}}(s) \geq V_t^{\pi^k}(s)$ for every $s$ and $k$. The base case $t = 0$ is trivial. For the step, at state $s$ the rule plays $a^{k^*} = \pi^{k^*(s)}(s)$, and
\begin{align}
V_{t+1}^{\pi^{\mathrm{sync}}}(s)
&= \E\Big[\phi(s, a^{k^*})^\top w_{\mathrm{test}} + \gamma\, V_t^{\pi^{\mathrm{sync}}}(s') \Big] \\
&\geq \E\Big[\phi(s, a^{k^*})^\top w_{\mathrm{test}} + \gamma\, V_t^{\pi^{k^*}}(s') \Big]
= V_{t+1}^{\pi^{k^*}}(s),
\end{align}
where the inequality applies the induction hypothesis at the successor state $s'$, and the equality uses that $\psi^{k^*}(s, a^{k^*})^\top w_{\mathrm{test}} = V^{\pi^{k^*}}(s)$ by validity. Since $k^*(s)$ maximizes $v^k(s)$ over the library, $V_{t+1}^{\pi^{k^*}}(s) = \max_k V_{t+1}^{\pi^k}(s) \geq V_{t+1}^{\pi^k}(s)$, which closes the induction. Taking $T \to \infty$ with $\gamma < 1$ gives the infinite-horizon statement by standard monotone convergence (the finite-horizon values are increasing in $T$ and bounded by the discounted feature sums). No assumption on the reward structure beyond validity, and none on the dynamics, is used: the argument never requires the transition kernel to factor or the features to decompose.

\subsection{Proof of Proposition~\ref{prop:reqsuffice}}

We show that Requirements~\ref{req:align} and \ref{req:valid} together imply the safety bound $V_{w_{\mathrm{test}}}^{\pi^{\mathrm{ind}}}(s) \geq \max_k V_{w_{\mathrm{test}}}^{\pi^k}(s)$ for every state. Consider the joint-GPI value of the library, $g_s(a) = \max_k \psi^k(s, a)^\top w_{\mathrm{test}}$, whose greedy policy $\pi^{\mathrm{gpi}}(s) \in \arg\max_a g_s(a)$ satisfies the single-agent guarantee $V_{w_{\mathrm{test}}}^{\pi^{\mathrm{gpi}}}(s) \geq \max_k V_{w_{\mathrm{test}}}^{\pi^k}(s)$ (Proposition~\ref{prop:gpi} of Appendix~\ref{app:single}, applied to the joint action space). It suffices to show that under the two requirements the independently composed policy $\pi^{\mathrm{ind}}$ of \eqref{eq:l2} selects a joint action in $\arg\max_a g_s(a)$ at every state and is scored by valid values.

Requirement~\ref{req:valid} states that each stored $\psi^k_i$ is the true successor feature of the composed joint policy, so every per-agent score $\psi^k_i(s, a_i)^\top w_i$ equals the corresponding component of the true joint value, and the joint score $\sum_i \psi^k_i(s, a_i)^\top w_i$ equals $\psi^k(s, a)^\top w_{\mathrm{test}}$; the values driving \eqref{eq:l2} are therefore the same values that define $g_s$, with no staleness. Requirement~\ref{req:align} states that the joint maximizer of $g_s$ is reachable by per-agent greedy choices, $A^{*}(s) = \prod_i A^{*}_i(s)$, so the coordinate-wise maximizers selected by \eqref{eq:l2} form a joint action in $\arg\max_a g_s(a)$ (a consistent tie-break resolves ties within the product set). Hence $\pi^{\mathrm{ind}}$ coincides with a joint-GPI policy scored by valid values, and the guarantee of Proposition~\ref{prop:gpi} transfers to it. The two requirements are exactly the two premises the single-agent argument needs once it is executed by $N$ decentralized maximizers rather than one: validity restores the fixed-dynamics premise at composition time, and alignment restores the single-decision-maker premise.

\subsection{The counterexample of Lemma~\ref{lem:l2unsafe}: full construction}

\paragraph{Task.}
Two agents, two stages, deterministic dynamics. Each agent has its own two-dimensional feature stream, and its reward depends only on its own action, so the team reward is fully separable: $r_i(s, a_i) = \phi_i(s, a_i)^\top w$ with a shared weight $w = (9, 8)$. The two coupling between the agents lives entirely in the transition: the second-stage state is $T(a_1, a_2)$ with $T(0,0) = T(1,1) = 0$ and $T(0,1) = T(1,0) = 1$. The feature values are collected in Table~\ref{tab:l2feat}.

\begin{table}[htbp]
\centering
\caption{Feature values of the Lemma~\ref{lem:l2unsafe} construction. Agent $i$'s reward at the first stage is $\phi_i(a_i)^\top w$, and at the second stage $\phi_i(s_1, a_i)^\top w$; only the transition depends on both agents' actions.}
\label{tab:l2feat}
\small
\begin{tabular}{llll}
\toprule
stage & agent 1 & agent 2 \\
\midrule
$s_0$, action $0$ & $(1,3)$ & $(0,7)$ \\
$s_0$, action $1$ & $(0,6)$ & $(6,5)$ \\
$s_1 = 0$, action $0$ & $(4,0)$ & $(7,7)$ \\
$s_1 = 0$, action $1$ & $(4,3)$ & $(4,6)$ \\
$s_1 = 1$, action $0$ & $(5,7)$ & $(3,4)$ \\
$s_1 = 1$, action $1$ & $(3,0)$ & $(2,2)$ \\
\bottomrule
\end{tabular}
\end{table}

\paragraph{Library.}
The library holds two entries, both suboptimal under $w$ (the joint optimum is $321$, computed below). Entry $\pi^0$ plays $(0,1)$ at $s_0$ and $(0,0)$ at the second stage in both states; entry $\pi^1$ plays $(0,1)$ at $s_0$ and $(0,1)$ or $(1,0)$ at the second stage. Their values are
\begin{align}
V^{\pi^0}(s_0) &= \big[(1,3) + (6,5) + (5,7) + (3,4)\big] \cdot w = (15,19) \cdot (9,8) = 287, \\
V^{\pi^1}(s_0) &= \big[(1,3) + (6,5) + (3,0) + (3,4)\big] \cdot w = (13,12) \cdot (9,8) = 213,
\end{align}
so the best library entry is worth $287$; the second-stage terms follow from $T(0,1) = 1$ and the entries' second-stage actions. Every library entry is strictly suboptimal: $287 < 321$.

\paragraph{The independent composition.}
For each agent we compute the per-agent snapshot value of each action, the expected feature sum when the agent takes the action while the teammate follows the entry and the team continues along the entry. At $s_0$, agent 1's values are $134$ for $a_1 = 0$ (achieved through entry $\pi^0$: teammate plays $1$, the team transitions to $s_1 = 1$ along $T(0,1)$, and agent 1 plays $0$ there, so $[(1,3) + (5,7)] \cdot w = 134$) and $84$ for $a_1 = 1$; agent 2's values are $175$ for $a_2 = 0$ (through $\pi^0$: teammate plays $0$, the team transitions to $s_1 = 0$ along $T(0,0)$, and agent 2 plays $0$ there, so $[(0,7) + (7,7)] \cdot w = 175$) and $153$ for $a_2 = 1$. Neither agent faces a tie. The independent rule therefore plays $(0,0)$ at $s_0$. The transition sends the team to $s_1 = 0$, where the per-agent values are the immediate feature rewards alone, $(36, 60)$ for agent 1 and $(119, 84)$ for agent 2, and the rule plays $(1,0)$. The composed trajectory delivers
\begin{equation}
\big[(1,3) + (0,7) + (4,3) + (7,7)\big] \cdot w = (12,20) \cdot (9,8) = 268 < 287 = \max_k V^{\pi^k}(s_0).
\end{equation}
The delivered per-agent shares are $93$ for agent 1 (rated $134$) and $175$ for agent 2 (rated $175$).

\paragraph{Mechanism.}
Agent 1's choice was justified by a value of $134$, the features it would accrue if agent 2 kept the library's action at $s_0$ and the team continued down the branch $s_1 = 1$ along entry $\pi^0$. Agent 2's independent choice flips the transition to $s_1 = 0$, and agent 1 delivers $93$ instead: its snapshot value described a teammate that is not the one the composition pairs it with. Agent 2's rating happened to be accurate, because agent 1's actual choice matches what the rating assumed; a single stale rating suffices to drag the team $19$ points below its own best library entry. Both agents acted exactly as their values told them to: the failure is not a selection error, it is a validity failure (Requirement~\ref{req:valid}). The centralized joint-GPI rule over the same library selects $(1,1)$ at $s_0$ (rated $297$ through entry $\pi^0$'s continuation) and $(1,0)$ at $s_1 = 0$, and delivers $[(0,6) + (6,5) + (4,3) + (7,7)] \cdot w = (17,21) \cdot (9,8) = 321$, the joint optimum: centralization protects the values, because the joint successor features remain valid descriptions of the joint policies they were measured for. The failure is therefore specific to decentralization.

\paragraph{Alignment holds.}
The joint actions optimal for the test objective are a singleton at every state: at $s_0$ the unique optimal joint action is $(1,1)$ (value $321$); at $s_1 = 0$ it is $(1,0)$ (value $179$, against $155$, $120$ and $144$ for the other three actions); at $s_1 = 1$ it is $(0,0)$ (value $160$, against $135$, $86$ and $61$). Singletons are trivially product sets, so the alignment condition of Requirement~\ref{req:align} holds at every state: no alignment story can explain the failure. What fails is Requirement~\ref{req:valid}: the per-agent values are snapshot quantities, measured against a specific version of the teammates, and recomposition invalidates them.

\paragraph{A companion instance: staleness without violation.}
The same mechanism need not always produce a violation; it always voids the guarantee. In a two-stage companion construction with joint features, weight $w = (0.6, 0.4)$, and a three-entry library whose best entry is worth $4.2$ (the other two worth $3.0$ and $3.2$), the alignment condition again holds at the test objective, yet the independent rule delivers $6.2$: above every library entry, so the safety inequality of Definition~\ref{def:safety} holds, but below the joint optimum $6.6$, which the centralized joint-GPI rule attains. The values that drove the composition were still wrong: agent 2's choice was justified by a snapshot value of $6.6$, the value it would have if agent 1 followed the third entry, and the composition delivers $6.2$ because agent 1 does not. The composition is safe by the letter of the definition and invalid by the mechanism that produces it.

\subsection{Proof of Proposition~\ref{prop:l2supermod}}

We prove the alignment claim that carries the proposition, then the value comparison. Throughout, $g_s(a) = \max_k \psi^k(s, a)^\top w_{\mathrm{test}}$ is the joint GPI value at state $s$.

\paragraph{Topkis lemma.} If $g_s$ is supermodular on the lattice $\{0,1\}^N$, its argmax set is a sublattice. For $a, b \in \arg\max g_s$, supermodularity gives $g_s(a \vee b) + g_s(a \wedge b) \geq g_s(a) + g_s(b) = 2g_s^*$, where $g_s^*$ is the maximum; since neither term can exceed $g_s^*$, both must equal it, so $a \vee b$ and $a \wedge b$ are also argmaxes, which is the defining property of a sublattice~\citep{topkis1998supermodularity}. The same conclusion holds when $g_s$ is supermodular on a rectangular sublattice containing all of its argmax, by restricting the argument to that sublattice.

\paragraph{Alignment.} The per-agent rule maximizes, per agent, the marginal value $h_i(a_i) = \max_k \psi^k_i(s, a_i)^\top w_i$. When the joint value decomposes over agents, $\psi^k(s, a)^\top w = \sum_i \psi^k_i(s, a_i)^\top w_i$, which holds under per-agent additive features, the per-agent greedy choices are the coordinate-wise maximizers of the decomposed library, and two facts combine: the argmax of $g_s$ is a sublattice by the lemma, and per-agent greedy execution with a consistent tie-break rule (the same resolution criterion applied by every agent, for example a shared index over entries) selects a vector in that sublattice. The composed policy then executes the joint GPI rule at every state, so $V^{\pi^{\mathrm{ind}}}(s) \geq \max_k V^{\pi^k}(s)$ by the standard generalized policy improvement argument (the same telescoping as in Appendix~\ref{app:proofs} above), and $\max_k V^{\pi^k}(s) = V^{\pi^{\mathrm{sync}}}(s)$ by Proposition~\ref{prop:l1safe}. This gives $V^{\pi^{\mathrm{ind}}}(s) \geq V^{\pi^{\mathrm{sync}}}(s)$.

In the general case without an additive decomposition, the per-agent marginals and the joint value are no longer the same objects, so the argument covers the additive (value-decomposed) regime, the setting in which the per-agent greedy choices realize the joint maximizer.

\subsection{Proof of Corollary~\ref{cor:cone}}

Let $\Delta\phi_d(a, b) = \phi_d(a \vee b) + \phi_d(a \wedge b) - \phi_d(a) - \phi_d(b)$ be the supermodularity gap of feature $d$ on the pair of joint actions $(a, b)$. Under the hypothesis, each library entry prices joint actions by the same feature-linear function up to an additive constant, $\psi^k(s, a)^\top w_{\mathrm{test}} = \phi(a)^\top w_{\mathrm{test}} + c_k(s)$ with $c_k(s)$ independent of $a$, since the per-pair differences $\psi^k(s, a)^\top w_{\mathrm{test}} - \psi^k(s, b)^\top w_{\mathrm{test}} = (\phi(a) - \phi(b))^\top w_{\mathrm{test}}$ do not depend on $k$. Taking the library maximum, $g_s(a) = \phi(a)^\top w_{\mathrm{test}} + \max_k c_k(s)$, so $g_s$ equals the feature-linear value up to a state-constant. The constant cancels in every second difference, so for every pair of joint actions
\begin{equation}
g_s(a \vee b) + g_s(a \wedge b) - g_s(a) - g_s(b) = \textstyle\sum_d w_d^{\mathrm{test}}\, \Delta\phi_d(a, b).
\end{equation}
Supermodularity of $g_s$ is nonnegativity of the left-hand side over all pairs, and membership $w_{\mathrm{test}} \in K_\phi$ is nonnegativity of the right-hand side over all pairs; these are the same inequalities, so $g_s$ is supermodular if and only if $w_{\mathrm{test}} \in K_\phi$, and Proposition~\ref{prop:l2supermod} applies exactly on the cone. The reasoning is not the loose ``maximum of supermodular functions is supermodular'' (false in general); it is the constant offset that makes the library maximum feature-linear. The sufficient condition stated in the corollary, every feature $\phi_d$ itself supermodular and $w \geq 0$, is immediate: each term $w_d\, \Delta\phi_d(a, b)$ is then nonnegative.

\subsection{Proof of Proposition~\ref{prop:l2factored}}

Assume the transition kernel factorizes per agent, $s'_i = f_i(s_i, a_i)$, and suppose first that the features also decompose per agent, so the per-agent successor feature $\psi^k_i(s, a_i)$ is the discounted feature stream of agent $i$'s own dynamics. Because the dynamics of agent $i$ never depend on the teammates, $\psi^k_i$ remains a valid description of agent $i$'s feature stream under any recomposition: whoever the other agents are, agent $i$'s marginal process is the same function of its own actions. Requirement~\ref{req:valid} therefore holds by construction, for every library and every recomposition, not only at the test objective. Requirement~\ref{req:align} is then governed by the structural condition of Proposition~\ref{prop:l2supermod}, and the value comparison follows as in that proposition: $V^{\pi^{\mathrm{ind}}}(s) \geq \max_k V^{\pi^k}(s) = V^{\pi^{\mathrm{sync}}}(s)$.

For coupled rewards, the cross-feature components of the per-agent values are measured against the library's teammate behavior, and under recomposition they are stale: the per-agent value is a marginalized quantity whose validity is no longer exact, so the guarantee need not extend beyond the decomposed-reward case. The channels are therefore not independent: factorization pays off exactly where the reward decomposes.

\section{MA-USFA: Training and Deployment}
\label{app:method}

This appendix gives the full training and deployment procedure of MA-USFA (Section~\ref{sec:method}), stated independently of any experimental domain. It collects the two-phase protocol, the pseudocode for both training and the deployment forward pass, and the two design choices that keep the composer small, and it closes by placing MA-USFA next to the two composition rules it is analyzed against.

\subsection{The two training phases}

The two phases of the MA-USFA protocol are summarized in Table~\ref{tab:phases}; the per-team-size budgets are reported with the controlled-domain setup (Appendix~\ref{app:toy}). Phase 1 trains the lower per-agent value layer $\tilde\psi_i$ of \eqref{eq:lowersf}, sampling tasks at training set (tasks from the homogeneous and heterogeneous scenarios) with the teammate-context conditioning, and freezes it; Phase 2 trains the upper-layer composer on top. Both phases build components of MA-USFA itself: the first produces the frozen value layer, and the second produces the composer stacked on top of it. The two phases correspond one-to-one with the two layers of Section~\ref{sec:method}.

\begin{table}[htbp]
\centering
\caption{The two-phase training protocol of MA-USFA. Each phase builds one layer of the method: Phase 1 the frozen per-agent value layer, Phase 2 the composer on top of it.}
\label{tab:phases}
\begin{tabular}{@{}p{3cm}p{3cm}p{2.7cm}p{3.0cm}@{}}
\toprule
Phase & Data & Object learned & Role \\
\midrule
1. context-conditioned value layer & $W_{\mathrm{homo}} \cup W_{\mathrm{hetero}}$ & $\tilde\psi_i(\cdot \mid z_i, w_{-i})$ & lower layer, Req.~\ref{req:valid} conditioning \\
2. learned composer & tasks needing correction & upper selectors $\Upsilon^\theta$ & upper layer, Lem.~\ref{lem:l2unsafe}, Props.~\ref{prop:l2supermod}, \ref{prop:l2factored} \\
\bottomrule
\end{tabular}
\end{table}

\subsection{Pseudocode}
\label{app:algo}

Algorithm~\ref{alg:train} states the training procedure and Algorithm~\ref{alg:deploy} the deployment-time forward pass, making concrete the two-phase protocol of Section~\ref{sec:method} and the notation of \eqref{eq:lowersf}. Training produces two artifacts that are never adapted per objective: the frozen per-agent value layer $\tilde\psi_i(s, a_i \mid z_i, w_{-i})$ and the composer $\Upsilon^\theta = (\Upsilon^\theta_1, \dots, \Upsilon^\theta_N)$. Deployment runs the composer on every test objective $w_{\mathrm{test}} = (w_1, \dots, w_N)$: because it is initialized at the independent transfer policy and trained only in value-increasing directions with the value layer frozen, on every objective it is at least as good as that policy and improves on it wherever a cross-agent correction helps. The composer is always the deciding rule; the fixed rules take no part in deployment.

\begin{algorithm}[htbp]
\caption{MA-USFA training (run once before deployment)}
\label{alg:train}
\begin{algorithmic}[1]
\Require feature map $\phi$; training tasks $W_{\mathrm{train}} = W_{\mathrm{homo}} \cup W_{\mathrm{hetero}}$; correction tasks $W_{\mathrm{corr}}$ (coupled dynamics or weights outside $K_\phi$); discount $\gamma$
\Statex \textbf{Phase 1 --- per-agent context-conditioned value layer} (lower layer, Eq.~\ref{eq:lowersf})
\For{each $w = (w_1,\dots,w_N)$ sampled at random from $W_{\mathrm{train}}$}
  \State sample policy encodings $z_i \sim D_z(\cdot \mid w_i)$; form the teammate context $w_{-i} = (w_j)_{j\neq i}$
  \State update each $\tilde\psi_i(s, a_i \mid z_i, w_{-i})$ toward $\phi_i + \gamma\,\tilde\psi_i(s', a_i' \mid z_i, w_{-i})$
\EndFor
\State freeze the value layer $\{\tilde\psi_i\}$
\Statex \textbf{Phase 2 --- learned composer} (upper layer, Lem.~\ref{lem:l2unsafe}, Props.~\ref{prop:l2supermod}, \ref{prop:l2factored})
\State initialize the composer head $\theta$ at zero \Comment{training starts exactly at the transfer policy}
\For{each $w \in W_{\mathrm{corr}}$ (weights outside the cone $K_\phi$, or coupled dynamics)}
  \State price candidates $q_i(z_i) \gets \tilde\psi_i(s, \cdot, z_i)^\top w_i$ for $z_i \in C_i$ using the frozen value layer
  \State select $g_i \gets \Upsilon^\theta_i\big(s,\, w,\, \{q_i(z_i)\}\big)$ for all $i$; execute the joint action; observe reward
  \State update $\theta$ by per-agent temporal difference (value layer frozen)
\EndFor
\State \Return frozen value layer $\{\tilde\psi_i\}$ and composer $\Upsilon^\theta$
\end{algorithmic}
\end{algorithm}

\begin{algorithm}[htbp]
\caption{MA-USFA deployment (single forward pass, no gradient updates)}
\label{alg:deploy}
\begin{algorithmic}[1]
\Require test objective $w_{\mathrm{test}} = (w_1,\dots,w_N)$; frozen $\{\tilde\psi_i\}$; composer $\Upsilon^\theta$; candidate sets $\{C_i\}$
\For{each decision state $s$}
  \State price candidates $q_i(z_i) \gets \tilde\psi_i(s, \cdot, z_i)^\top w_i$ for every $z_i \in C_i$ \Comment{one dot product per candidate}
  \State $a_i \gets$ action of entry $g_i = \Upsilon^\theta_i\big(s,\, w_{\mathrm{test}},\, \{q_i(z_i)\}\big)$ for all $i$ \Comment{composer selects on every objective}
  \State execute the joint action $(a_1,\dots,a_N)$
\EndFor
\end{algorithmic}
\end{algorithm}

\paragraph{Implementation choices.} Three choices make the composer small and stable, and they are shared across both experimental domains. The composer head is initialized at zero, so training starts exactly at the transfer policy; the value layer is frozen while the composer trains, protecting the library from drift; and for large teams the composer attends over the interaction graph rather than the full joint state, a one-hop graph attention pass per selector~\citep{velickovic2018graph}, which is what makes the $196$-agent traffic network of Appendix~\ref{app:tsc} feasible. The first two hold in every run; the third is the neighbor-limited instantiation used at city scale (Appendix~\ref{app:tsc}).

\paragraph{Extension: deployment without a composer.} One extension lies outside the method proper. If the composer cannot be trained at all, for instance under a learning budget too small to fit it, deployment can use a fixed rule in its place, and the analysis of Section~\ref{sec:theory} says which one: synchronized composition by default, and the independent rule only where its two requirements hold. This chooses among the rules we analyze rather than the one we propose; MA-USFA itself needs no such choice.

\subsection{Composition rules studied as baselines}

MA-USFA is analyzed against two fixed composition rules, which are objects of study rather than components of the method. Table~\ref{tab:spectrum} places the three side by side, listing their mechanism, the set of joint policies each can produce, and the safety status established in Section~\ref{sec:theory}.

\begin{table}[htbp]
\centering
\caption{The coupling spectrum of composition rules. Composition space is the set of joint policies a rule can produce. Safety is with respect to \eqref{eq:safety}.}
\label{tab:spectrum}
\begin{tabular}{llll}
\toprule
Rule & Mechanism & Composition space & Safety \\
\midrule
Sync & synchronized shared index & $K$ library entries & unconditional (Prop.~\ref{prop:l1safe}) \\
Indep & independent per-agent argmax & product $K^N$ & conditional (Prop.~\ref{prop:l2supermod}, \ref{prop:l2factored}) \\
MA-USFA & learned composer & product $K^N$ & learned (Sec.~\ref{sec:method}) \\
\bottomrule
\end{tabular}
\end{table}

\section{Controlled Domain: SFWorld}
\label{app:toy}

This appendix reports the controlled-domain experiment in full: the environment and setup, the crossover result behind the main-body table, and the team-size and library-coverage sweeps.

\subsection{Setup}

\paragraph{Environment.} SFWorld is a $5 \times 5$ grid with $N \in \{2, 3, 4, 5\}$ agents, five actions (four moves and stay), horizon $40$, discount $\gamma = 0.95$. The features $\phi_i(s')$ of agent $i$ are dense smooth bumps over $K$ resource cells, shaped after the congestion signals of the traffic domain; the reward of agent $i$ is $r_i = \phi_i(s')^\top w_i - 2.0 \cdot \mathrm{blocked}_i$, a task weight times the features minus a collision penalty that lives outside the feature basis, so any closed-form pricing $\psi \cdot w$ is structurally blind to it. The coupling parameter $\kappa$ governs conflicts: when two agents target the same cell, the conflict blocks all but one of them with probability $\kappa$. At $\kappa = 0$ the dynamics factorize exactly, the free region of Proposition~\ref{prop:l2factored} is in force, and the collision rate is zero at every $\kappa$ on $A_{\mathrm{distinct}}$; as $\kappa$ grows, each agent's effective kernel depends on the teammates' actions, which is transition coupling. Two task families separate the channels (Fig.~\ref{fig:sfworld}): $A_{\mathrm{distinct}}$, where the agents' optimal regions are disjoint so coupling never binds, and $B_{\mathrm{overlap}}$, where the agents contend for a shared region so the same weights create collisions whenever the joint policy hybridizes.

\begin{figure}[htbp]
\centering
\includegraphics[width=0.82\textwidth]{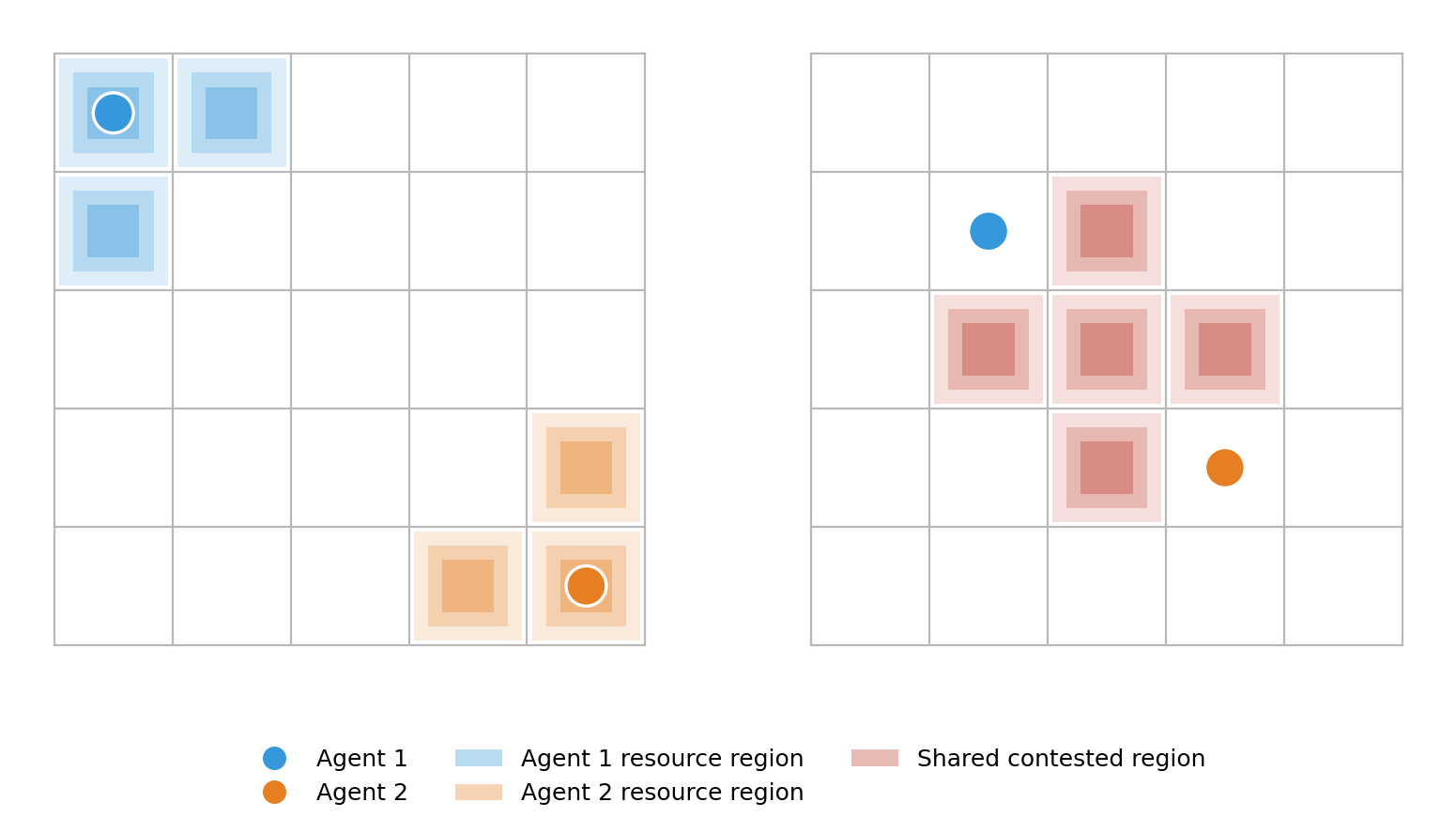}
\caption{The SFWorld scene and its two task families, shown for $N = 2$ on the $5 \times 5$ grid. Each agent earns reward by occupying cells in its own resource region, whose value under a task is $\phi_i(s')^\top w_i$; a collision penalty, outside the feature basis, is charged whenever two agents target the same cell. Left ($A_{\mathrm{distinct}}$): the two resource regions are disjoint. Right ($B_{\mathrm{overlap}}$): a single contested region is valuable to both agents. The coupling probability $\kappa$ scales how often a shared-cell conflict blocks an agent.}
\label{fig:sfworld}
\end{figure}

\paragraph{Library.} The library holds $K$ policies specialized to corner objectives. The main configuration uses the four corners of the weight simplex ($K = 4$); the library-content sweep (Fig.~\ref{fig:kcoverage}) varies both the size and the content: the four corners alone (P4), the corners plus the contested joint action $(0,0)$ (P4$+$(0,0)), plus both $(0,0)$ and $(1,1)$ (P4$+$(0,0)$+$(1,1)), the corners plus four copies of one entry (P4$+$4 same), and libraries built around $(0,0)$ alone or with one or two additional entries. The team-size sweep uses $K = 4$ for $N \leq 4$ and $K = 8$ for $N = 5$.

\paragraph{Learning.} Successor features are trained with the USFA protocol: the TD target is the feature vector $\phi$, the weight plays the three roles of behavior anchor, library anchor ($z \sim D_z(\cdot \mid w)$), and test-time pricing vector. Training budgets scale with the team size: $8{,}000$ episodes for $N = 2$, $6{,}000$ for $N = 3$, $5{,}000$ for $N = 4$, $4{,}000$ for $N = 5$; the composer is trained for $1{,}000$ episodes at $N = 2$, $2{,}000$ at $N = 3$, $1{,}600$ at $N = 4$, $1{,}400$ at $N = 5$; per-task retraining runs $4{,}000$ episodes. Evaluation averages $120$ rollouts per configuration. The composer is warm-started from the transfer policy (zero-initialized correction head) and trained with per-agent temporal-difference updates with the feature backbone frozen; of the three composer instantiations, an attention head, a message-passing head, and a QMIX-style mixing head, the three variants are interchangeable in this domain because the optimal deferral is static, so the paper reports the best of the three. The coupling sweep uses $\kappa \in \{0, 0.25, 0.5, 0.75, 1.0\}$ in the main matrix and $\kappa \in \{0, 0.5, 1.0\}$ in the team-size sweep.

\paragraph{Baselines.} The synchronized rule is the composition rule of \eqref{eq:l1}; the independent rule is the composition rule of \eqref{eq:l2}; joint-GPI is the centralized generalized policy improvement rule over the joint action space, representing prior work on multi-agent GPI (computed only for $N \leq 3$, where the joint space can be enumerated; the joint-GPI entries for $N \geq 4$ in Table~\ref{tab:nfull} are missing by design); per-task retraining optimizes the penalty-inclusive reward from scratch and is the oracle ceiling. All closed-form operators use successor features trained on the harvest features only, as theory prescribes; the collision penalty is deliberately outside the feature basis so that no closed-form rule can price it. The two-phase training protocol of MA-USFA is given in Appendix~\ref{app:method} (Table~\ref{tab:phases}); the phase budgets per team size are those reported under Learning above.

\subsection{The crossover}

Table~\ref{tab:mainbody} in the body reports the $N = 2$ main matrix. Fig.~\ref{fig:crossover} plots its $B_{\mathrm{overlap}}$ block, so the crossover of the two fixed rules and the tracking of MA-USFA against retraining are visible as curves in the coupling parameter $\kappa$.

\begin{figure}[t!]
\centering
\includegraphics[width=0.72\textwidth]{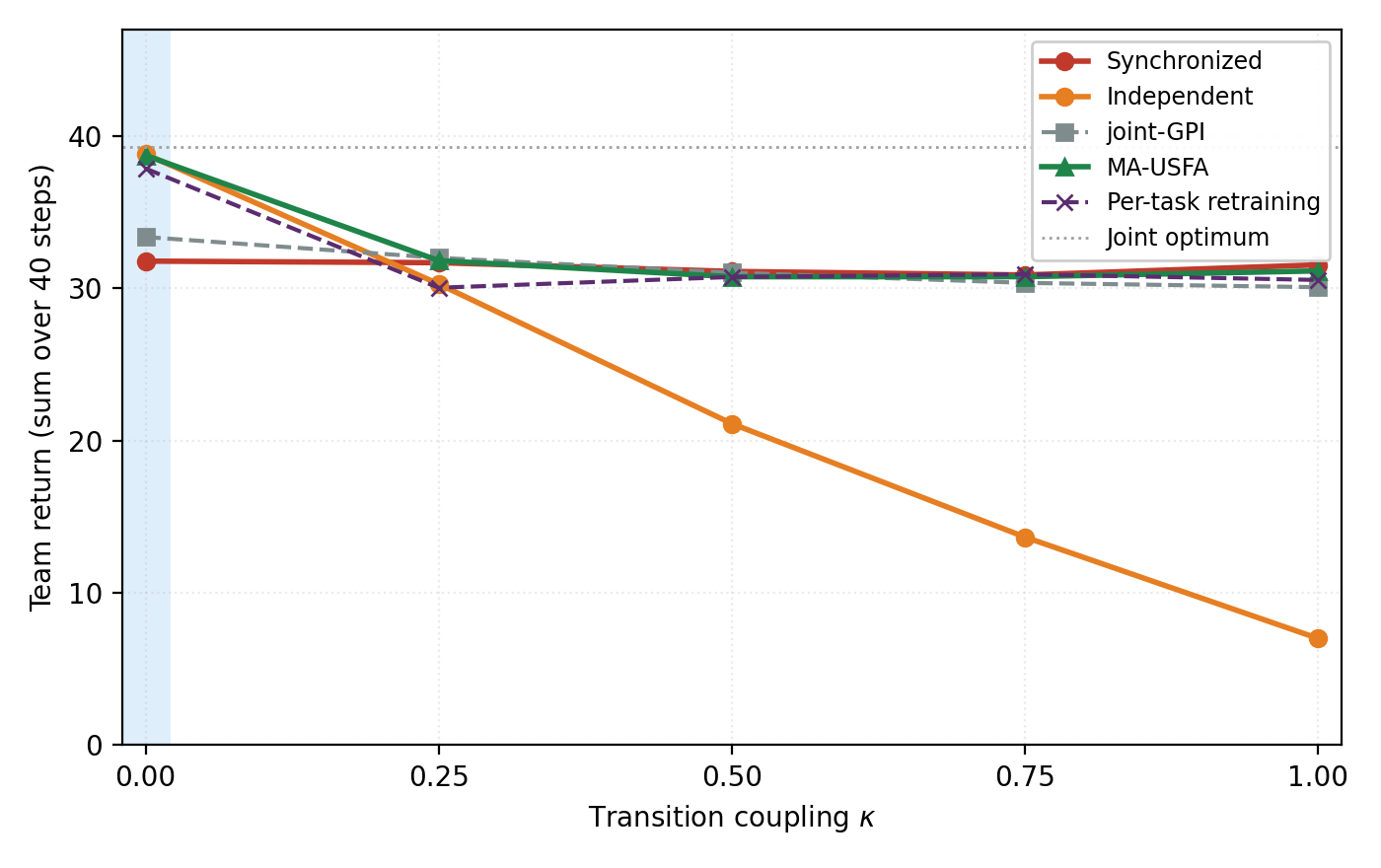}
\caption{Crossover experiment on $B_{\mathrm{overlap}}$ ($N = 2$, four-entry corner library), the $B_{\mathrm{overlap}}$ block of Table~\ref{tab:mainbody} in visual form. Curves show team return (summed over the $40$-step horizon) of the three composition rules of Section~\ref{sec:prelim} (Synchronized, Independent, MA-USFA), the joint-GPI baseline of prior work, and per-task retraining, as the transition coupling $\kappa$ (the probability that two agents targeting the same cell collide) increases from $0$ to $1$. The dotted horizontal line marks the per-task joint optimum; the shaded strip at $\kappa = 0$ marks the factorized regime.}
\label{fig:crossover}
\end{figure}

\subsection{Team-size sweep}

Table~\ref{tab:nfull} and Fig.~\ref{fig:nsweep} sweep the team from $N = 2$ to $5$ on both task families at $\kappa \in \{0, 0.5, 1.0\}$. The paid-region collapse of the independent rule deepens sharply with team size, from $7.07$ at $N = 2$ to $-37.00$ at $\kappa = 1.0$ as its collision rate approaches $0.6$; the synchronized rule stays flat and safe at every size, and MA-USFA recovers the retraining level at every $N$ and $\kappa$. On $A_{\mathrm{distinct}}$ the collision rate is zero at all sizes and the independent rule is free, the null control of the design.

\begin{table}[t!]
\centering
\caption{Team-size sweep, corner library ($K = 4$ for $N \leq 4$, $K = 8$ for $N = 5$), all three coupling levels. Team return over $40$ steps, mean $\pm$ standard deviation over three evaluation seeds. The joint-GPI baseline is omitted ($-$) for $N \geq 4$ because the joint action space can no longer be enumerated. The Average row under each team size gives each method's mean over the six task and $\kappa$ settings at that $N$.}
\label{tab:nfull}
\small
\resizebox{\textwidth}{!}{%
\begin{tabular}{@{}llrrrrrrr@{}}
\toprule
$N$ & task & $\kappa$ & Sync & Indep & joint-GPI & MA-USFA & retrain & $\mathrm{coll_{ind}}$ \\
\midrule
2 & $A_{\mathrm{distinct}}$ & 0.00 & $23.42{\scriptstyle\pm0.34}$ & $\mathbf{38.73}{\scriptstyle\pm0.47}$ & $23.74{\scriptstyle\pm0.29}$ & $38.40{\scriptstyle\pm0.53}$ & $\underline{38.69}{\scriptstyle\pm0.48}$ & 0.000 \\
2 & $A_{\mathrm{distinct}}$ & 0.50 & $23.37{\scriptstyle\pm0.16}$ & $\underline{38.23}{\scriptstyle\pm0.01}$ & $23.63{\scriptstyle\pm0.14}$ & $38.17{\scriptstyle\pm0.22}$ & $\mathbf{38.35}{\scriptstyle\pm0.10}$ & 0.000 \\
2 & $A_{\mathrm{distinct}}$ & 1.00 & $23.44{\scriptstyle\pm0.15}$ & $\underline{38.31}{\scriptstyle\pm0.27}$ & $23.80{\scriptstyle\pm0.23}$ & $38.25{\scriptstyle\pm0.18}$ & $\mathbf{38.48}{\scriptstyle\pm0.16}$ & 0.000 \\
2 & $B_{\mathrm{overlap}}$ & 0.00 & $31.52{\scriptstyle\pm0.26}$ & $\mathbf{38.77}{\scriptstyle\pm0.14}$ & $32.59{\scriptstyle\pm0.20}$ & $\underline{38.44}{\scriptstyle\pm0.13}$ & $37.56{\scriptstyle\pm0.25}$ & 0.000 \\
2 & $B_{\mathrm{overlap}}$ & 0.50 & $\mathbf{31.27}{\scriptstyle\pm0.09}$ & $21.04{\scriptstyle\pm0.94}$ & $30.65{\scriptstyle\pm0.14}$ & $30.94{\scriptstyle\pm0.30}$ & $\underline{31.08}{\scriptstyle\pm0.25}$ & 0.225 \\
2 & $B_{\mathrm{overlap}}$ & 1.00 & $\mathbf{31.24}{\scriptstyle\pm0.16}$ & $7.07{\scriptstyle\pm0.20}$ & $28.62{\scriptstyle\pm0.60}$ & $\underline{31.18}{\scriptstyle\pm0.18}$ & $30.33{\scriptstyle\pm0.31}$ & 0.369 \\
\addlinespace
\multicolumn{3}{@{}l}{Average, $N{=}2$} & 27.38 & 30.36 & 27.17 & $\mathbf{35.90}$ & $\underline{35.75}$ & 0.099 \\
\midrule
3 & $A_{\mathrm{distinct}}$ & 0.00 & $26.14{\scriptstyle\pm0.19}$ & $\mathbf{58.22}{\scriptstyle\pm0.48}$ & $27.70{\scriptstyle\pm0.21}$ & $57.98{\scriptstyle\pm0.51}$ & $\underline{58.14}{\scriptstyle\pm0.49}$ & 0.000 \\
3 & $A_{\mathrm{distinct}}$ & 0.50 & $25.79{\scriptstyle\pm0.06}$ & $\underline{57.35}{\scriptstyle\pm0.35}$ & $25.61{\scriptstyle\pm0.17}$ & $57.07{\scriptstyle\pm0.43}$ & $\mathbf{57.52}{\scriptstyle\pm0.11}$ & 0.000 \\
3 & $A_{\mathrm{distinct}}$ & 1.00 & $25.47{\scriptstyle\pm0.16}$ & $56.74{\scriptstyle\pm0.81}$ & $25.13{\scriptstyle\pm0.91}$ & $\underline{57.54}{\scriptstyle\pm0.65}$ & $\mathbf{57.68}{\scriptstyle\pm0.19}$ & 0.000 \\
3 & $B_{\mathrm{overlap}}$ & 0.00 & $43.64{\scriptstyle\pm0.29}$ & $\mathbf{58.14}{\scriptstyle\pm0.38}$ & $48.23{\scriptstyle\pm0.07}$ & $\underline{57.91}{\scriptstyle\pm0.37}$ & $57.85{\scriptstyle\pm0.36}$ & 0.000 \\
3 & $B_{\mathrm{overlap}}$ & 0.50 & $\underline{43.19}{\scriptstyle\pm0.29}$ & $22.38{\scriptstyle\pm0.83}$ & $38.92{\scriptstyle\pm0.12}$ & $\mathbf{43.60}{\scriptstyle\pm0.21}$ & $42.86{\scriptstyle\pm0.24}$ & 0.329 \\
3 & $B_{\mathrm{overlap}}$ & 1.00 & $42.07{\scriptstyle\pm0.18}$ & $-6.52{\scriptstyle\pm0.40}$ & $31.80{\scriptstyle\pm1.10}$ & $\mathbf{42.63}{\scriptstyle\pm0.14}$ & $\underline{42.22}{\scriptstyle\pm0.17}$ & 0.504 \\
\addlinespace
\multicolumn{3}{@{}l}{Average, $N{=}3$} & 34.38 & 41.05 & 32.90 & $\mathbf{52.79}$ & $\underline{52.71}$ & 0.139 \\
\midrule
4 & $A_{\mathrm{distinct}}$ & 0.00 & $29.05{\scriptstyle\pm0.29}$ & $\mathbf{77.43}{\scriptstyle\pm0.45}$ & --- & $77.05{\scriptstyle\pm0.49}$ & $\underline{77.29}{\scriptstyle\pm0.46}$ & 0.000 \\
4 & $A_{\mathrm{distinct}}$ & 0.50 & $28.51{\scriptstyle\pm0.07}$ & $\mathbf{76.52}{\scriptstyle\pm0.12}$ & --- & $75.93{\scriptstyle\pm0.35}$ & $\underline{76.46}{\scriptstyle\pm0.73}$ & 0.000 \\
4 & $A_{\mathrm{distinct}}$ & 1.00 & $27.64{\scriptstyle\pm0.35}$ & $75.19{\scriptstyle\pm1.40}$ & --- & $\underline{75.99}{\scriptstyle\pm0.72}$ & $\mathbf{76.40}{\scriptstyle\pm0.86}$ & 0.000 \\
4 & $B_{\mathrm{overlap}}$ & 0.00 & $55.84{\scriptstyle\pm0.49}$ & $\mathbf{77.78}{\scriptstyle\pm0.40}$ & --- & $77.31{\scriptstyle\pm0.41}$ & $\underline{77.48}{\scriptstyle\pm0.39}$ & 0.000 \\
4 & $B_{\mathrm{overlap}}$ & 0.50 & $\underline{55.00}{\scriptstyle\pm0.22}$ & $23.08{\scriptstyle\pm0.26}$ & --- & $\mathbf{55.03}{\scriptstyle\pm0.09}$ & $54.40{\scriptstyle\pm0.60}$ & 0.349 \\
4 & $B_{\mathrm{overlap}}$ & 1.00 & $53.43{\scriptstyle\pm0.30}$ & $-21.48{\scriptstyle\pm0.06}$ & --- & $\mathbf{54.52}{\scriptstyle\pm0.69}$ & $\underline{53.79}{\scriptstyle\pm0.44}$ & 0.556 \\
\addlinespace
\multicolumn{3}{@{}l}{Average, $N{=}4$} & 41.58 & 51.42 & --- & $\mathbf{69.31}$ & $\underline{69.30}$ & 0.151 \\
\midrule
5 & $A_{\mathrm{distinct}}$ & 0.00 & $49.89{\scriptstyle\pm0.22}$ & $\mathbf{98.75}{\scriptstyle\pm0.41}$ & --- & $98.22{\scriptstyle\pm0.42}$ & $\underline{98.67}{\scriptstyle\pm0.40}$ & 0.000 \\
5 & $A_{\mathrm{distinct}}$ & 0.50 & $48.72{\scriptstyle\pm0.42}$ & $96.34{\scriptstyle\pm0.47}$ & --- & $\underline{96.78}{\scriptstyle\pm0.32}$ & $\mathbf{97.45}{\scriptstyle\pm0.20}$ & 0.000 \\
5 & $A_{\mathrm{distinct}}$ & 1.00 & $47.62{\scriptstyle\pm0.29}$ & $92.86{\scriptstyle\pm0.75}$ & --- & $\underline{95.94}{\scriptstyle\pm1.00}$ & $\mathbf{96.40}{\scriptstyle\pm1.44}$ & 0.003 \\
5 & $B_{\mathrm{overlap}}$ & 0.00 & $50.42{\scriptstyle\pm0.15}$ & $\mathbf{97.51}{\scriptstyle\pm0.40}$ & --- & $96.70{\scriptstyle\pm0.51}$ & $\underline{96.99}{\scriptstyle\pm0.43}$ & 0.000 \\
5 & $B_{\mathrm{overlap}}$ & 0.50 & $48.94{\scriptstyle\pm0.22}$ & $23.38{\scriptstyle\pm1.20}$ & --- & $\underline{67.23}{\scriptstyle\pm0.14}$ & $\mathbf{67.44}{\scriptstyle\pm0.29}$ & 0.364 \\
5 & $B_{\mathrm{overlap}}$ & 1.00 & $47.29{\scriptstyle\pm0.18}$ & $-37.00{\scriptstyle\pm0.31}$ & --- & $\mathbf{66.34}{\scriptstyle\pm0.35}$ & $\underline{66.11}{\scriptstyle\pm0.14}$ & 0.601 \\
\addlinespace
\multicolumn{3}{@{}l}{Average, $N{=}5$} & 48.81 & 61.97 & --- & $\underline{86.87}$ & $\mathbf{87.18}$ & 0.161 \\
\bottomrule
\end{tabular}}
\end{table}

\begin{figure}[htbp]
\centering
\includegraphics[width=0.95\textwidth]{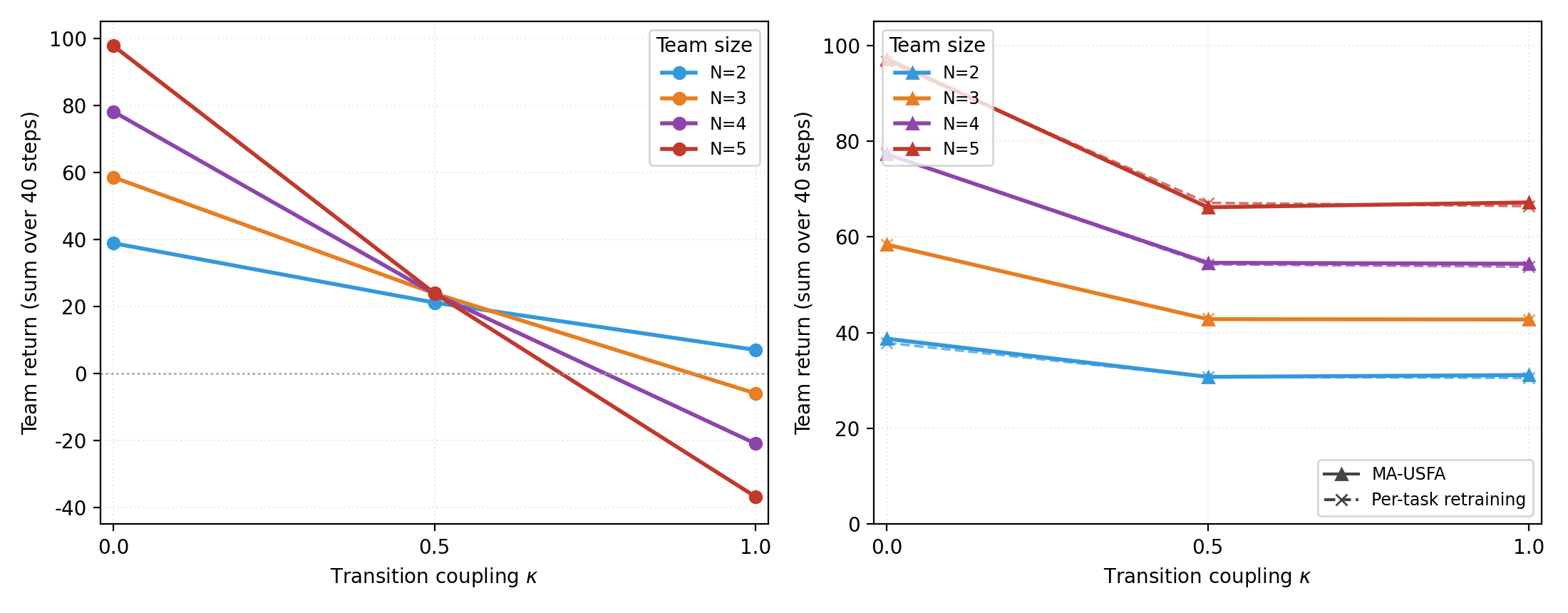}
\caption{Team-size sweep on $B_{\mathrm{overlap}}$ for $N = 2$ to $5$ at $\kappa \in \{0, 0.5, 1\}$, where $\kappa$ is the collision probability and team return is summed over the $40$-step horizon. Left: team return of independent composition, one curve per team size $N$. Right: MA-USFA (solid) and per-task retraining (dashed), one color per $N$. Synchronized composition is omitted for legibility.}
\label{fig:nsweep}
\end{figure}

\subsection{Library-coverage sweep}

Table~\ref{tab:kfull} and Fig.~\ref{fig:kcoverage} vary the library content on $B_{\mathrm{overlap}}$ ($N = 2$), from a single entry to the corner library augmented with the contested joint action $(0,0)$, and show that the value of coverage is regime-dependent. The mechanism is that Proposition~\ref{prop:l1safe} assumes the anchor price $\psi^k(s,a)^\top w_{\mathrm{test}}$ equals the true value of entry $k$ under $w_{\mathrm{test}}$, which is exact inside the linear feature model but becomes a transfer-blind estimate once the deployed reward carries a component outside the feature basis. In the free region ($\kappa = 0$) adding the contested vertex is pure gain, lifting the synchronized rule from $31.52$ to $38.72$, near the joint optimum; in the paid region ($\kappa = 1.0$) the same addition locks the team onto the entry the anchor overprices and crashes it to $7.08$ while the library still contains entries worth $31.24$, a signed gap of $-24.16$, and enlarging the library further does not cure it. The learned composer is immune to library content, returning the same $31.18$ across every row of the paid region.

\begin{table}[t!]
\centering
\caption{Library-content sweep ($N = 2$, $B_{\mathrm{overlap}}$). Team return over $40$ steps, mean $\pm$ standard deviation over three evaluation seeds. P4 is the four-entry corner library; the other rows add or replace entries around the contested joint action $(0,0)$. $\max_k V^k$ is the value of the best available source and the last column is the signed gap of the synchronized rule to it. The final row averages each method over all library variants and coupling levels.}
\label{tab:kfull}
\small
\resizebox{\textwidth}{!}{%
\begin{tabular}{@{}lrrrrrrrr@{}}
\toprule
library & $\kappa$ & Sync & Indep & MA-USFA & retrain & $\mathrm{coll_{ind}}$ & $\max_k V^k$ & Sync $-$ $\max_k V^k$ \\
\midrule
$(0,0)$ only & 0.00 & $\underline{38.72}{\scriptstyle\pm0.16}$ & $\mathbf{38.77}{\scriptstyle\pm0.14}$ & $38.19{\scriptstyle\pm0.09}$ & $37.56{\scriptstyle\pm0.25}$ & 0.000 & 38.72 & $+0.00$ \\
$(0,0)$ only & 0.50 & $21.18{\scriptstyle\pm0.34}$ & $21.04{\scriptstyle\pm0.94}$ & $\underline{28.34}{\scriptstyle\pm0.42}$ & $\mathbf{31.08}{\scriptstyle\pm0.25}$ & 0.225 & 21.18 & $+0.00$ \\
$(0,0)$ only & 1.00 & $7.08{\scriptstyle\pm0.05}$ & $7.07{\scriptstyle\pm0.20}$ & $\mathbf{31.18}{\scriptstyle\pm0.18}$ & $\underline{30.33}{\scriptstyle\pm0.31}$ & 0.369 & 7.08 & $+0.00$ \\
$(0,0)+(1,2)$ & 0.00 & $\underline{38.72}{\scriptstyle\pm0.16}$ & $\mathbf{38.77}{\scriptstyle\pm0.14}$ & $38.19{\scriptstyle\pm0.09}$ & $37.56{\scriptstyle\pm0.25}$ & 0.000 & 38.72 & $+0.00$ \\
$(0,0)+(1,2)$ & 0.50 & $21.18{\scriptstyle\pm0.34}$ & $21.04{\scriptstyle\pm0.94}$ & $\underline{28.34}{\scriptstyle\pm0.42}$ & $\mathbf{31.08}{\scriptstyle\pm0.25}$ & 0.225 & 21.18 & $+0.00$ \\
$(0,0)+(1,2)$ & 1.00 & $7.08{\scriptstyle\pm0.05}$ & $7.07{\scriptstyle\pm0.20}$ & $\mathbf{31.18}{\scriptstyle\pm0.18}$ & $\underline{30.33}{\scriptstyle\pm0.31}$ & 0.369 & 7.08 & $+0.00$ \\
$(0,0)+(1,2)+(2,1)$ & 0.00 & $\underline{38.72}{\scriptstyle\pm0.16}$ & $\mathbf{38.77}{\scriptstyle\pm0.14}$ & $38.19{\scriptstyle\pm0.09}$ & $37.56{\scriptstyle\pm0.25}$ & 0.000 & 38.72 & $+0.00$ \\
$(0,0)+(1,2)+(2,1)$ & 0.50 & $21.18{\scriptstyle\pm0.34}$ & $21.04{\scriptstyle\pm0.94}$ & $\underline{28.34}{\scriptstyle\pm0.42}$ & $\mathbf{31.08}{\scriptstyle\pm0.25}$ & 0.225 & 21.18 & $+0.00$ \\
$(0,0)+(1,2)+(2,1)$ & 1.00 & $7.08{\scriptstyle\pm0.05}$ & $7.07{\scriptstyle\pm0.20}$ & $\mathbf{31.18}{\scriptstyle\pm0.18}$ & $\underline{30.33}{\scriptstyle\pm0.31}$ & 0.369 & 14.83 & $-7.75$ \\
P4 corners & 0.00 & $31.52{\scriptstyle\pm0.26}$ & $\mathbf{38.77}{\scriptstyle\pm0.14}$ & $\underline{38.19}{\scriptstyle\pm0.09}$ & $37.56{\scriptstyle\pm0.25}$ & 0.000 & 31.52 & $+0.00$ \\
P4 corners & 0.50 & $\mathbf{31.27}{\scriptstyle\pm0.09}$ & $21.04{\scriptstyle\pm0.94}$ & $28.34{\scriptstyle\pm0.42}$ & $\underline{31.08}{\scriptstyle\pm0.25}$ & 0.225 & 31.27 & $+0.00$ \\
P4 corners & 1.00 & $\mathbf{31.24}{\scriptstyle\pm0.16}$ & $7.07{\scriptstyle\pm0.20}$ & $\underline{31.18}{\scriptstyle\pm0.18}$ & $30.33{\scriptstyle\pm0.31}$ & 0.369 & 31.24 & $+0.00$ \\
P4$+$(0,0) & 0.00 & $\underline{38.72}{\scriptstyle\pm0.16}$ & $\mathbf{38.77}{\scriptstyle\pm0.14}$ & $38.19{\scriptstyle\pm0.09}$ & $37.56{\scriptstyle\pm0.25}$ & 0.000 & 38.72 & $+0.00$ \\
P4$+$(0,0) & 0.50 & $21.18{\scriptstyle\pm0.34}$ & $21.04{\scriptstyle\pm0.94}$ & $\underline{28.34}{\scriptstyle\pm0.42}$ & $\mathbf{31.08}{\scriptstyle\pm0.25}$ & 0.225 & 31.27 & $-10.10$ \\
P4$+$(0,0) & 1.00 & $7.08{\scriptstyle\pm0.05}$ & $7.07{\scriptstyle\pm0.20}$ & $\mathbf{31.18}{\scriptstyle\pm0.18}$ & $\underline{30.33}{\scriptstyle\pm0.31}$ & 0.369 & 31.24 & $-24.16$ \\
P4$+$(0,0)$+$(1,1) & 0.00 & $\underline{38.72}{\scriptstyle\pm0.16}$ & $\mathbf{38.77}{\scriptstyle\pm0.14}$ & $38.19{\scriptstyle\pm0.09}$ & $37.56{\scriptstyle\pm0.25}$ & 0.000 & 38.72 & $+0.00$ \\
P4$+$(0,0)$+$(1,1) & 0.50 & $21.18{\scriptstyle\pm0.34}$ & $21.04{\scriptstyle\pm0.94}$ & $\underline{28.34}{\scriptstyle\pm0.42}$ & $\mathbf{31.08}{\scriptstyle\pm0.25}$ & 0.225 & 31.27 & $-10.10$ \\
P4$+$(0,0)$+$(1,1) & 1.00 & $7.08{\scriptstyle\pm0.05}$ & $7.07{\scriptstyle\pm0.20}$ & $\mathbf{31.18}{\scriptstyle\pm0.18}$ & $\underline{30.33}{\scriptstyle\pm0.31}$ & 0.369 & 31.24 & $-24.16$ \\
P4$+$4 same & 0.00 & $\underline{38.72}{\scriptstyle\pm0.16}$ & $\mathbf{38.77}{\scriptstyle\pm0.14}$ & $38.19{\scriptstyle\pm0.09}$ & $37.56{\scriptstyle\pm0.25}$ & 0.000 & 38.72 & $+0.00$ \\
P4$+$4 same & 0.50 & $21.18{\scriptstyle\pm0.34}$ & $21.04{\scriptstyle\pm0.94}$ & $\underline{28.34}{\scriptstyle\pm0.42}$ & $\mathbf{31.08}{\scriptstyle\pm0.25}$ & 0.225 & 31.27 & $-10.10$ \\
P4$+$4 same & 1.00 & $7.08{\scriptstyle\pm0.05}$ & $7.07{\scriptstyle\pm0.20}$ & $\mathbf{31.18}{\scriptstyle\pm0.18}$ & $\underline{30.33}{\scriptstyle\pm0.31}$ & 0.369 & 31.24 & $-24.16$ \\
\midrule
\multicolumn{2}{@{}l}{Average (all rows)} & 23.61 & 22.29 & $\underline{32.57}$ & $\mathbf{32.99}$ & 0.198 & --- & $-5.26$ \\
\bottomrule
\end{tabular}}
\end{table}

\paragraph{The synchronized anchor must be valid.} Proposition~\ref{prop:l1safe} assumes that $\psi^k(s,a)^\top w_{\mathrm{test}}$ equals the true value of entry $k$ under $w_{\mathrm{test}}$; inside the linear feature model this is exact. When the deployed reward contains a component outside the feature basis, the estimate becomes a transfer-blind anchor that prices only what the features saw during training. Table~\ref{tab:kfull} and Fig.~\ref{fig:kcoverage} show the consequence: at $\kappa = 1$ the synchronized rule locks the team onto the contested entry $(0,0)$ and falls to $7.08$ while the library still contains entries worth $31.24$, a signed gap of $-24.16$, and enlarging the library does not cure it, because the contested entry is exactly the one the anchor overprices. The learned composer is immune to library content ($31.18$ across every row of the paid region).

\begin{figure}[t!]
\centering
\includegraphics[width=0.9\textwidth]{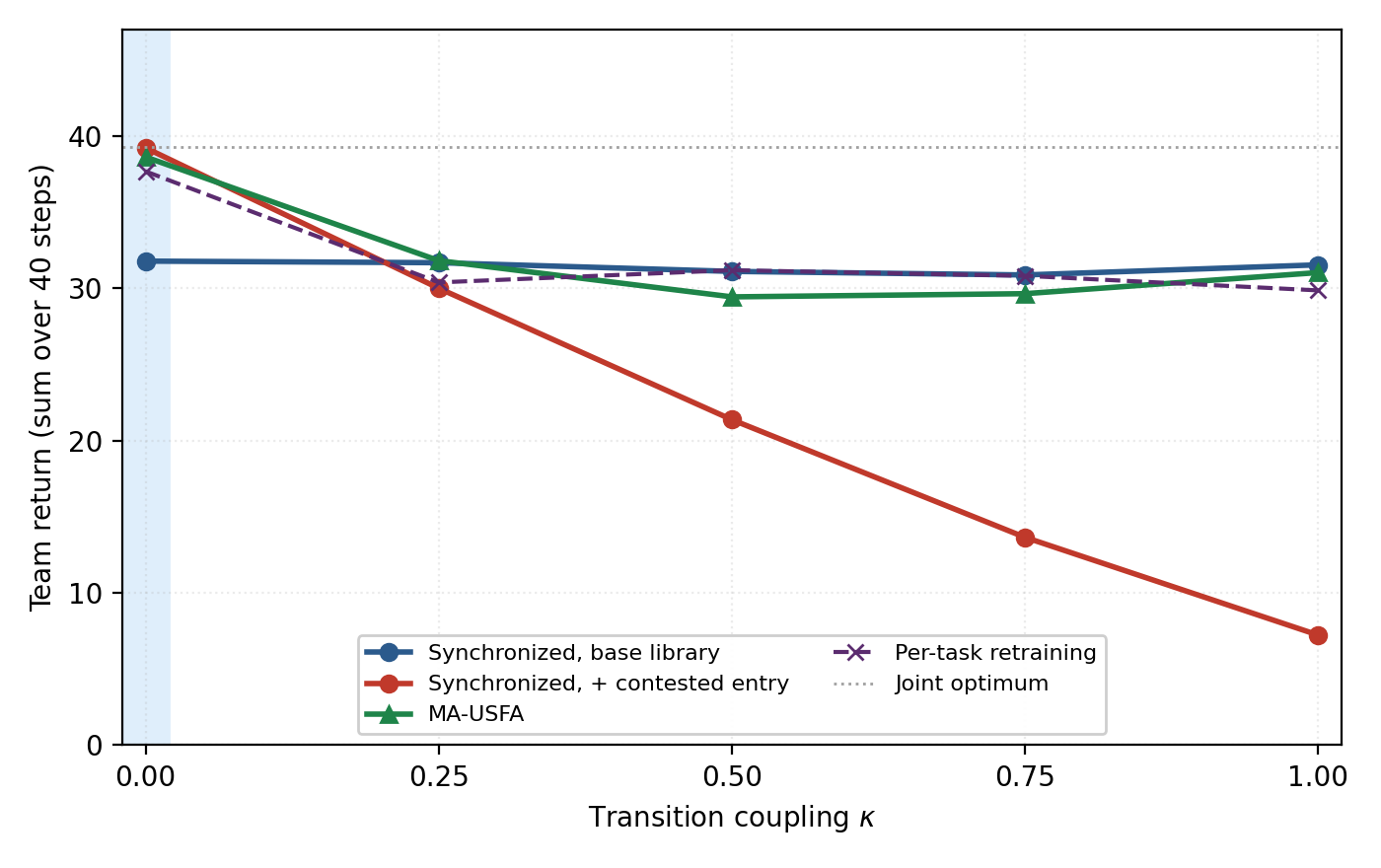}
\caption{Library-content sweep on $B_{\mathrm{overlap}}$ ($N = 2$). Team return (summed over the $40$-step horizon) versus transition coupling $\kappa$ for synchronized composition under two libraries, the base four-entry corner library and the same library with the contested joint entry $(0,0)$ added, compared against MA-USFA and per-task retraining. The dotted horizontal line marks the per-task joint optimum; the shaded strip at $\kappa = 0$ marks the factorized regime.}
\label{fig:kcoverage}
\end{figure}

\section{City-Scale Domain: Traffic Signal Control}
\label{app:tsc}

This appendix reports the traffic-signal experiment in full, in the same order as the controlled domain: the environment and setup, the main result, and the robustness sweeps.

\subsection{Setup}

\paragraph{Network and coupling.} The experiment runs on a Manhattan grid of $28 \times 7$ intersections, $196$ signal agents, one per intersection. Road cells are shared between adjacent intersections, so spillback couples the state of each intersection to the decisions of its upstream neighbors: the transition coupling is endogenous, with no free parameter to control, and the factorization condition of Proposition~\ref{prop:l2factored} fails by construction. The feature model is the project's composite traffic reward, which is already of the form $\phi^\top w$: $\phi = [\mathrm{queue}, \mathrm{wait}, \mathrm{pressure}, \mathrm{speed}]_{\mathrm{norm}}$, four normalized congestion features, so the composite reward is literally a successor-feature target and its weight vector is the task. All four features keep a fixed physical sign (queue, wait, and pressure are costs, speed is a gain); a task is a choice of the relative magnitudes on these four terms.

\paragraph{Homogeneous and heterogeneous tasks.} The two tasks differ only in whether the target weight is shared across the city or varies by intersection. In the \emph{homogeneous} task every intersection optimizes one shared weight vector, so the $196$ agents pursue a single common goal. 
In the \emph{heterogeneous} task each intersection draws its own weight independently, keeping the fixed signs but sampling the four magnitudes uniformly in $[0.5, 1.5]$ times their defaults, so different corners of the city weigh queue, waiting, pressure, and speed differently at the same moment; the closed-form transfer must then serve every node's own $w$ simultaneously, and the composer can additionally condition on its neighbors' weights. The reported metrics are the task return $\sum_i \phi_i^\top w_i$, mean travel time over completed trips, throughput (completed trips), mean queue, mean waiting time, and mean speed.

\paragraph{Composer.} The composer is the neighbor-limited instantiation of Section~\ref{sec:method}: a one-hop graph attention selector per intersection, $u_i = \psi_i^\top w_i + \rho \tanh(\mathrm{head}(\mathrm{GAT}_i(\mathrm{context})))$ with $\rho = 4.0$, zero-initialized so that training starts exactly at the transfer policy, with the feature backbone frozen; global information reaches each decision only through direct neighbors, the strict locality constraint a $196$-agent network imposes, following the neighbor-limited cooperative pattern of CoLight~\citep{wei2019colight}. The composer budget is $200$ episodes, short relative to the toy but expensive per episode at this scale.

\subsection{Main result}

\paragraph{Results.} Table~\ref{tab:tsc} reports the full numbers, and the picture is consistent across both task families: MA-USFA is best on every metric, and its margins over the two fixed rules are an order of magnitude larger than the gap between those rules, whose fine ordering we do not claim. On the homogeneous task the composer reaches return $2165$ against $2009$ (independent) and $1993$ (synchronized), throughput $518$ against $368$ and $356$, and mean waiting $6.60$ s against $10.35$ and $11.95$ s; on the heterogeneous task it reaches return $1650$ against $1515$ and $1527$, throughput $483$ against $337$ and $352$, and waiting $5.48$ s against $8.57$ and $8.50$ s. One entry needs reading with care: the composer's homogeneous travel time ($288.2$ s, the longest in its block) is a measurement artifact rather than a regression, since travel time is averaged over completed trips and the composer clears roughly $150$ extra long-distance trips that the fixed rules leave stranded, which lengthens the mean it is credited with. 

\begin{table}[htbp]
\centering
\caption{Traffic signal control on the Manhattan $28 \times 7$ network ($196$ intersections), on the homogeneous and heterogeneous tasks. Return is the task objective $\sum_i \phi_i^\top w_i$ summed over the horizon; the remaining columns are physical traffic metrics. Bold and underline mark the best and second-best method within each task family.}
\label{tab:tsc}
\resizebox{\textwidth}{!}{%
\begin{tabular}{llrrrrrr}
\toprule
Task & Method & Return & Travel (s) & Throughput & Queue & Wait (s) & Speed (m/s) \\
\midrule
Homo & Sync (anchor) & 1993 & 280.1 & 356 & 3.13 & 11.95 & 6.26 \\
Homo & Indep (transfer) & \underline{2009} & \textbf{274.7} & \underline{368} & \underline{2.94} & \underline{10.35} & \underline{6.36} \\
Homo & MA-USFA (composer) & \textbf{2165} & \underline{288.2} & \textbf{518} & \textbf{2.54} & \textbf{6.60} & \textbf{6.96} \\
\addlinespace
Hetero & Indep (transfer) & 1515 & \underline{273.1} & 337 & 2.81 & 8.57 & 6.49 \\
Hetero & Sync (anchor) & \underline{1527} & 276.2 & \underline{352} & \underline{2.73} & \underline{8.50} & \underline{6.52} \\
Hetero & MA-USFA (composer) & \textbf{1650} & \textbf{270.2} & \textbf{483} & \textbf{2.27} & \textbf{5.48} & \textbf{6.98} \\
\bottomrule
\end{tabular}}
\end{table}

\subsection{Robustness sweeps}

Two sweeps on the homogeneous task support the headline numbers (Table~\ref{tab:tsc-sweep}). The composer already clears both fixed rules after $100$ fine-tuning episodes (return $2156$, throughput $456$), and further episodes buy only small refinements ($2165$ at $200$, $2174$ at $500$), so the $200$-episode budget used above sits near the plateau rather than being a compute-heavy outlier. Enlarging the synchronized rule's source library from one to five entries does not help it and eventually hurts: return holds at $2032$ for $K \in \{1,2,3\}$, dips to $2028$ at $K = 4$, and falls to $1993$ at $K = 5$, because added specialists rarely become the winning shared anchor and the last, contested entry drags the team down in the coupled regime, the same coverage liability the controlled domain isolates (Appendix~\ref{app:toy}). The independent rule ($2009$) sits between the small and large synchronized libraries, and the composer clears every configuration.

\begin{table}[htbp]
\centering
\caption{Supplementary traffic sweeps on the homogeneous task. Top: the MA-USFA composer as a function of the fine-tuning budget in episodes. Bottom: the synchronized rule as a function of source-library size $K$, with the independent rule listed for reference. Columns are as in Table~\ref{tab:tsc}}
\label{tab:tsc-sweep}
\resizebox{\textwidth}{!}{%
\begin{tabular}{llrrrrrr}
\toprule
Sweep & Setting & Return & Travel (s) & Throughput & Queue & Wait (s) & Speed (m/s) \\
\midrule
Composer budget & MA-USFA, $100$ ep. & 2156 & 268.2 & 456 & 2.71 & 6.73 & 6.90 \\
Composer budget & MA-USFA, $200$ ep. & 2165 & 288.2 & 518 & 2.54 & 6.60 & 6.96 \\
Composer budget & MA-USFA, $500$ ep. & 2174 & 277.5 & 493 & 2.40 & 6.70 & 6.94 \\
\addlinespace
Library size & Sync, $K = 1$ & 2032 & 275.4 & 357 & 2.75 & 8.69 & 6.52 \\
Library size & Sync, $K = 3$ & 2032 & 275.4 & 357 & 2.75 & 8.69 & 6.52 \\
Library size & Sync, $K = 4$ & 2028 & 272.2 & 345 & 2.79 & 8.85 & 6.51 \\
Library size & Sync, $K = 5$ & 1993 & 280.1 & 356 & 3.13 & 11.95 & 6.26 \\
Library size & Indep (reference) & 2009 & 274.7 & 368 & 2.94 & 10.35 & 6.36 \\
\bottomrule
\end{tabular}}
\end{table}

\section{Discussions}
\label{app:usfa}

The proposed MA-USFA sits inside a lineage whose generality we rely on at three points, and this appendix states the facts precisely. Universal value function approximators condition the value model directly on the goal, $V(s, g)$, and interpolate across goals by function approximation~\citep{schaul2015universal}. Universal successor feature approximators combine the two ideas in a model $\tilde\psi(s, a, z, w)$ with a policy axis $z$ and a task axis $w$: $z$ indexes the policies in the library, and $w$ shapes the behavior distribution during training and prices the library at test time through a dot product~\citep{borsa2019universal}. With a single goal as the candidate set, USFA reduces to UVFA; with a finite library it reduces to successor features with generalized policy improvement. The operating mode we import into a team is the one USFA was designed for: train once over a distribution of objectives, and at deployment evaluate dot products only, with no per-task adaptation.

MA-USFA uses the same machinery per agent. Agent $i$'s value model $\tilde\psi_i(s, a_i \mid z_i, w_{-i})$ keeps $z_i$ as its own policy axis, whose self-referential TD target $a'_i = \arg\max_b \tilde\psi_i(s', b, z_i)^\top z_i$ has as its fixed point the optimal policy for the task $z_i$; the task axis is carried by the weights, linear in the value, appearing only at test time. The conditioning on the teammates' weights $w_{-i}$ is the new ingredient: each teammate's weight indexes a cluster of teammate policies, so the model learns the expected successor features over the teammate behavior distribution rather than a snapshot against one fixed version of the teammates, which is the architectural answer to Requirement~\ref{req:valid} of Section~\ref{sec:independent}. Two structural facts make the reuse exact rather than approximate. First, when the features are per-agent additive, the joint successor feature of a synchronized entry decomposes as $\psi^{\mathrm{joint},k} = \sum_i \psi^k_i$, so the same network serves both the synchronized anchor and the per-agent library; the sum is the value the synchronized rule prices. Second, the composer is initialized at the independent rule with the value layer frozen, and its correction is trained only to raise team value, so on every objective MA-USFA is at least as good as per-agent USFA and coincides with it wherever the independent rule is already optimal, without any runtime test.